\documentclass{article}

\usepackage{microtype}
\usepackage{graphicx}
\usepackage{subcaption}
\usepackage{booktabs} 

\usepackage{hyperref}

\usepackage[preprint]{icml2026}

\usepackage{amsmath}
\usepackage{amssymb}
\usepackage{mathtools}
\usepackage{amsthm}

\usepackage[capitalize,noabbrev]{cleveref}

\theoremstyle{plain}
\newtheorem{theorem}{Theorem}[section]

\theoremstyle{definition}

\theoremstyle{remark}

\usepackage[textsize=tiny]{todonotes}

\usepackage{xcolor} 
\usepackage{blindtext} 
\usepackage{lipsum} 
\usepackage{amssymb}
\usepackage{amsthm} 
\usepackage{enumitem} 
\usepackage{algorithm}
\usepackage{algorithmic}
\usepackage{amsmath,amssymb}

\theoremstyle{plain}

\icmltitlerunning{AGP-MARL}

\begin{document}

\twocolumn[
  \icmltitle{Action-Graph Policies: Learning Action Co-dependencies in \\Multi-Agent Reinforcement Learning}



  \icmlsetsymbol{equal}{*} 

  \begin{icmlauthorlist}
    \icmlauthor{Nikunj Gupta}{usc}
    \icmlauthor{James Zachary Hare}{arl}
    \icmlauthor{Jesse Milzman}{arl}
    \icmlauthor{Rajgopal Kannan}{aro}
    \icmlauthor{Viktor Prasanna}{usc}
  \end{icmlauthorlist}

  \icmlaffiliation{usc}{University of Southern California, Los Angeles, CA, USA}
  \icmlaffiliation{arl}{DEVCOM Army Research Laboratory, Adelphi, MD, USA}
  \icmlaffiliation{aro}{DEVCOM ARL Army Research Office, Los Angeles, CA, USA}

  \icmlcorrespondingauthor{Nikunj Gupta}{nikunj@usc.edu}




  \icmlkeywords{Machine Learning, ICML}

  \vskip 0.3in
]



\printAffiliationsAndNotice{}  

\begin{abstract}
\label{sec:abstract} 

Coordinating actions is the most fundamental form of cooperation in multi-agent reinforcement learning (MARL). Successful decentralized decision-making often depends not only on good individual actions, but on selecting compatible actions across agents to synchronize behavior, avoid conflicts, and satisfy global constraints. In this paper, we propose Action Graph Policies (AGP), that model dependencies among agents' available action choices. It constructs, what we call, \textit{coordination contexts}, that enable agents to condition their decisions on global action dependencies. Theoretically, we show that AGPs induce a strictly more expressive joint policy compared to fully independent policies and can realize coordinated joint actions that are provably more optimal than greedy execution even from centralized value-decomposition methods. Empirically, we show that AGP achieves 80-95\% success on canonical coordination tasks with partial observability and anti-coordination penalties, where other MARL methods reach only 10-25\%. We further demonstrate that AGP consistently outperforms these baselines in diverse multi-agent environments.

\end{abstract}

\section{Introduction} 
\label{sec:Introduction} 
Multi-agent systems face coordination challenges across diverse domains: swarms of autonomous drones must coordinate surveillance patterns without collisions \cite{marek2025collision,alqudsi2025uav,tahir2019swarms}, warehouse robots must allocate tasks to avoid congestion \cite{zhen2025optimizing,keith2024review}, and distributed sensor networks must jointly select which nodes transmit to conserve energy \cite{li2022applications}. In each case, success requires agents to satisfy global constraints, for example, exactly $K$ agents must be active, no two robots may select the same resource, or total transmissions must remain below a threshold, while making decentralized decisions based only on local observations. These coordination problems share a common structure: optimal joint behavior depends not merely on each agent selecting a good individual action, but on the compatibility of actions across the team. When exactly three vehicles must enter an intersection simultaneously for smooth traffic flow, or when a formation requires precisely two agents to occupy point positions, independent decision-making systematically fails. Yet in decentralized multi-agent reinforcement learning (MARL), the dominant learning paradigms continue to assume complete agent independence even during centralized training \cite{zhang2021multi,oroojlooy2023review}. 
This imposes a policy-side restriction, fundamentally limiting which coordinated joint actions can be represented or enforced.  




In many coordination-critical applications, it is both reasonable and desirable to minimally relax strict independence at execution time, allowing limited, structured coordination. Examples include shared situational context in robotics teams \cite{wang2022distributed,siddiqua2024information,senaratne2025framework}, lightweight coordination primitives in networked control systems \cite{liu2025survey,mason2024multi}, or pre-decision aggregation in multi-robot planning \cite{ebert2022distributed}. These settings motivate policy architectures that move beyond fully independent execution, without resorting to fully centralized or sequential control. 

In this paper, we propose \textbf{Action-Graph Policies (AGP)}, a policy architecture that enables richer coordination by explicitly modeling action-level relations among agents. Rather than treating agents as the fundamental units of coordination, AGP adopts an action-centric representation in which individual \((\text{agent}, \text{action})\) pairs are modeled as nodes in a learned graph. The key insight is that coordination constraints such as ``exactly $K$ agents act'' or ``no two agents select the same resource'' are naturally expressed as relationships between specific action choices across agents. AGP constructs embeddings for each such pair and applies graph attention over these action nodes to learn which combinations are compatible or conflicting. 
Each agent's policy then conditions on its local observation and a learned ``coordination context" which summarizes how its available actions participate in the global coordination structure. This design enables the policy to represent joint behaviors that are unreachable under independent factorization. 

This paper makes four primary contributions. 
\textbf{First}, we provide a formal analysis of the policy-side bottleneck, characterizing coordination problems where independent policies are provably suboptimal and showing that value decomposition cannot overcome this limitation. 
\textbf{Second}, we introduce Action-Graph Policies that model action-level dependencies to enable coordinated joint behavior with minimal relaxation of decentralized execution. 
\textbf{Third}, we present a theoretical characterization of the joint policy class induced by AGP, proving that it strictly generalizes independent decentralized policies and admits joint policies that are provably closer (in forward KL divergence) to optimal centralized policies than any independently factorized execution. 
\textbf{Finally}, we empirically validate AGP on coordination benchmarks: in canonical games with partial observability and anti-coordination penalties, AGP achieves $80$-$95\%$ success while other MARL methods achieve only $10$-$25\%$, and in diverse multi-agent domains, AGP exceeds strong MARL baselines.

\section{Preliminaries and Related Works} 
\label{sec:Preliminaries} 
\paragraph{Model.} 
We model a cooperative MARL problem as a decentralized partially observable Markov decision process 
(Dec-POMDP) \cite{oliehoek2016concise,bernstein2002complexity}, defined by the tuple
\begin{equation}
\mathcal{M} = \langle \mathcal{I}, \mathcal{S}, \{\mathcal{A}_i\}_{i \in \mathcal{I}}, 
\mathcal{T}, \mathcal{R}, \{\mathcal{O}_i\}_{i \in \mathcal{I}}, \Omega, \gamma \rangle.
\end{equation}
Here, $\mathcal{I}=\{1,\dots,N\}$ is the set of agents, $\mathcal{S}$ is 
the global state space, $\mathcal{A}_i$ is the discrete action space of 
agent $i$, $\mathcal{A}=\prod_{i \in \mathcal{I}} \mathcal{A}_i$ is the joint action space, 
$\mathcal{T}(s' \mid s, \mathbf{a})$ is the transition function, 
$\mathcal{R}(s,\mathbf{a})$ is the shared reward, $\mathcal{O}_i$ is the 
observation space of agent $i$, $\Omega(\mathbf{o} \mid s,\mathbf{a})$ 
is the observation function, and $\gamma \in (0,1)$ is the discount factor.
At each timestep $t$, the environment is in an unobserved state 
$s_t \in \mathcal{S}$. Each agent $i$ receives a local observation 
$o_{i,t} \in \mathcal{O}_i$, selects an action 
$a_{i,t} \in \mathcal{A}_i$, and we denote the joint observation and joint action as
$\mathbf{o}_t = (o_{1,t},\dots,o_{N,t})$ and 
$\mathbf{a}_t = (a_{1,t},\dots,a_{N,t})$. 
The joint action $\mathbf{a}_t$ produces a shared reward
$r_t=\mathcal{R}(s_t,\mathbf{a}_t)$ and a next state 
$s_{t+1} \sim \mathcal{T}(\cdot \mid s_t,\mathbf{a}_t)$. 
A stochastic joint policy $\pi$ is a conditional distribution
$
\pi(\mathbf{a} \mid \mathbf{o})
= \pi(a_1,\dots,a_N \mid o_1,\dots,o_N),
$
mapping joint observations $\mathbf{o}$ to joint actions $\mathbf{a}$.
The objective is to maximize the expected discounted return
$J(\pi) = \mathbb{E}_{\pi,\mathcal{T}}\Big[ \sum_{t=0}^{\infty} \gamma^t r_t \Big]$.
We denote the joint action-value function by $Q^\pi(s,\mathbf{a})$ and the
state-value function by $V^\pi(s)$.

\paragraph{Joint Policies and Independent-Policy Factorization.} 
A joint policy specifies a distribution over joint actions conditioned on 
joint observations,
\begin{equation}
\pi(\mathbf{a}\mid \mathbf{o}) \in \Delta(\mathcal{A}_1 \times \cdots \times \mathcal{A}_N),
\end{equation}
where $\Delta(\cdot)$ denotes the probability simplex over a finite set,
$\mathbf{o}$ and $\mathbf{a}$.
This joint policy fully characterizes the coordinated behavior of the multi-agent 
system. Directly representing or learning $\pi(\mathbf{a}\mid \mathbf{o})$ 
is generally intractable due to the exponential growth of the joint action 
space with the number of agents. For this reason, most decentralized MARL 
methods approximate the joint policy using a product of per-agent policies, $\pi_i(a_i \mid o_i)$, 
yielding the independent policy form
\begin{equation}
\label{eq:independent}
\pi(\mathbf{a}\mid \mathbf{o}) = \prod_{i=1}^N \pi_i(a_i \mid o_i).
\end{equation}
This factorization is referred to as the independent global optimum 
(IGO) assumption \cite{zhang2021fop}. We denote by $\Pi_{\mathrm{ind}}$ 
the class of all joint policies admitting this factorization. Importantly, 
$\Pi_{\mathrm{ind}}$ defines a strict subset of all possible joint 
distributions over actions.

\paragraph{Value Decomposition Methods.} 
To compensate for the loss of expressiveness induced by independent policies, 
many MARL methods adopt a centralized training, decentralized execution 
(CTDE) paradigm \cite{kraemer2016multi,amato2024introduction} and introduce 
coordination on the value side. In this framework, agents still execute 
independent policies of the form $\pi_i(a_i \mid o_i)$, but learning is 
guided by a centralized action-value function that evaluates joint actions. 
Value-decomposition methods \cite{sunehag2017value,rashid2020monotonic,rashid2020weighted,son2019qtran,wang2020qplex} instantiate 
this idea by learning per-agent utility functions $Q_i(o_i,a_i)$ together 
with a mixing function that combines them into a joint value,
\begin{equation}
Q_{\mathrm{tot}}(\mathbf{o},\mathbf{a})
= f\big(Q_1(o_1,a_1), \dots, Q_N(o_N,a_N)\big),
\end{equation}
where $f$ is designed to ensure that decentralized greedy action selection 
with respect to each $Q_i$ is consistent with maximizing $Q_{\mathrm{tot}}$. 
For instance, VDN \cite{sunehag2017value} uses a simple additive mixing, 
while QMIX \cite{rashid2020monotonic} and its variants learn a more flexible, 
monotonic mixing network. More expressive value-based methods introduce 
structured interactions between agents' actions through coordination graphs 
\cite{bohmer2020deep, li2020deepimplicit, wang2022contextaware, gupta2025deep, duan2024group, kang2022non, yang2022self, gupta2025tiger, liu2025survey}. 
In this setting, agents are represented as nodes in a graph, and edges encode
pairwise utility functions between agents. The joint value is decomposed as
\begin{equation}
Q_{\mathrm{tot}}(\mathbf{o},\mathbf{a})
= \tfrac{1}{|V|}\sum_{i} Q_i(o_i,a_i)
+ \tfrac{1}{|\mathcal{E}|}\sum_{(i,j)\in\mathcal{E}}
Q_{ij}(o_i,a_i,o_j,a_j).
\end{equation}
where $\mathcal{E}$ denotes the set of edges in the coordination graph.
Although these methods improve coordination during learning, their
representational power remains inherently constrained and susceptible to
\emph{relative overgeneralization} \cite{bohmer2020deep,panait2006biasing}, i.e.,
agents overgeneralize toward locally optimal actions under additive or monotonic
factorizations, leading to miscoordination. Joint values expressed in this form
cannot capture general higher-order coordination without becoming intractable
\cite{castellini2019representational}. See Appendix~\ref{app:pairwise} for a
detailed discussion.

\paragraph{Action-dependent policies.}
Action-dependent policies relax independent factorization by allowing an agent’s
decision to condition on other agents’ actions or action distributions. In
general, such policies take the form
\begin{equation}
\pi_i(a_i \mid o_i, a_{-i})
\quad \text{or} \quad
\pi(\mathbf a \mid \mathbf o)
= \prod_{i=1}^N \pi_i(a_i \mid o_i, \mathbf a_{-i}),
\end{equation}
thereby enlarging the realizable joint policy class beyond $\Pi_{\mathrm{ind}}$.
Prior work instantiates this idea through autoregressive \cite{wen2022multi,fu2022revisiting}
or DAG-based factorizations \cite{ruan2022gcs,liu2023deep}, which generate joint
actions sequentially under a fixed or learned ordering. Recent work on policy
factorization \cite{wang2023more,zhang2021fop,fu2022revisiting,li2025agentmixer}
uses dependent policies during centralized training to guide independent policy
learning and execution. While these approaches improve expressiveness, they
typically require explicit agent orderings, introduce asymmetries, or rely on
conditioning on preceding actions, complicating decentralized execution.

\section{Analysis} 
\label{sec:Analysis}

As defined in Section~2, decentralized execution under IGO assumption \cite{zhang2021fop} restricts policies to the factorized class $\Pi_{\mathrm{ind}}$. Equivalently, for any joint observation $\mathbf o$, the set of joint action distributions that can be executed at test time is limited
to product-form distributions
\begin{equation}
\mathcal P_{\mathrm{ind}}(\mathbf o)
= \left\{ p(\mathbf a) : p(\mathbf a) = \prod_{i=1}^N p_i(a_i) \right\}.
\end{equation}
Throughout this section, we analyze the consequences of this execution constraint. In particular, we distinguish between joint policies that are representable or learnable during centralized training, and those that are executable under decentralized control. Our results characterize coordinated behaviors that fall outside $\mathcal P_{\mathrm{ind}}(\mathbf o)$ and are therefore inaccessible to independently executed policies, regardless of the training procedure or value-function expressiveness. 


\paragraph{Independent Execution Is Structurally Suboptimal.} We first establish that restricting execution to $\Pi_{\mathrm{ind}}$ can be fundamentally suboptimal.

\begin{theorem}[Independent policies are not universally optimal]
\label{thm:ind_suboptimal}
There exist cooperative Dec-POMDPs with deterministic optimal joint policies $\pi^*$ such that $\pi^* \notin \Pi_{\mathrm{ind}}$.
\end{theorem}

\textbf{Proof.} See Appendix~\ref{app:ind}. \qed

This result is not about optimization difficulty or credit assignment, but about representational limits. Independent policies cannot enforce hard coordination constraints that depend on the joint action $\mathbf a$ as a whole, such as requiring exactly $K$ agents to act, selecting a unique committing agent, or satisfying global action patterns. Such constraints define subsets of $\mathcal A$ that are not decomposable into independent marginals and so cannot be represented by policies in $\Pi_{\mathrm{ind}}$.

\paragraph{Value Decomposition Does Not Remove the Policy-Side Bottleneck.} Value-decomposition addresses coordination by coupling agents during training through a centralized joint action-value function $Q_{\mathrm{tot}}(\mathbf o,\mathbf a)$. However, decentralized execution remains independent: agents select actions greedily with respect to local utilities,
\begin{equation}
\mathbf a(\mathbf o)
= \bigl(\arg\max_{a_1} Q_1(o_1,a_1), \dots, \arg\max_{a_N} Q_N(o_N,a_N)\bigr).
\end{equation}
So, execution is still restricted to joint actions in $\mathcal P_{\mathrm{ind}}(\mathbf o)$.

\begin{theorem}[Value--policy mismatch]
\label{thm:vd_mismatch}
There exist cooperative Dec-POMDPs and centralized joint action-value functions $Q_{\mathrm{tot}}(\mathbf o,\mathbf a)$ such that:
\begin{enumerate}[noitemsep, topsep=0pt, partopsep=0pt, parsep=0pt]
    \item $Q_{\mathrm{tot}}$ admits an exact representation by a value-decomposition model, and
    \item the optimal joint action
    \[
    \mathbf a^*(\mathbf o) = \arg\max_{\mathbf a} Q_{\mathrm{tot}}(\mathbf o,\mathbf a)
    \]
    cannot be realized by decentralized greedy execution under independent policies.
\end{enumerate}
\end{theorem}

\textbf{Proof.} See Appendix~\ref{app:vd}. \qed

Theorem~\ref{thm:vd_mismatch} highlights a fundamental mismatch between value expressiveness and policy executability. Even when the centralized value function correctly identifies the optimal joint action, decentralized greedy execution is constrained to select actions independently. Increasing the expressiveness of $Q_{\mathrm{tot}}$ alone therefore does not enlarge the set of joint behaviors that can be executed.

\paragraph{The Policy-Side Bottleneck.} Consider that $N$ agents with binary actions must ensure that exactly $K$ select action $1$ (e.g., exactly $2$ out of $5$ drones ascend to high altitude for communication relay). Under an independent policy where each agent selects action $1$ with probability $p$, the probability of achieving exactly $K$ selections is
\begin{equation}
\Pr\!\left(\sum_{i=1}^N a_i = K\right) = \binom{N}{K} p^K (1-p)^{N-K} 
\label{eq:optimal_ind}
\end{equation}
Even at the optimal choice $p^\star = K/N$, this probability is strictly less than $1$. For $N=5$ and $K=2$, the maximum achievable probability is only $0.3456$, implying that the team fails nearly $65\%$ of the time despite optimal independent policies. A centralized controller, by contrast, can deterministically select exactly $K$ agents and achieve $100\%$ success. This performance gap does not arise from learning difficulty or credit assignment, but from a representational limitation: the optimal joint action distribution, which concentrates probability mass on configurations with exactly $K$ active agents, cannot be factored into a product of independent marginals.

We refer to this as a \textit{policy-side bottleneck}. Variants of this analysis have been recognized in different forms in prior work, for instance, through assumptions such as independent global maxima or optimality, or relative overgeneralization, but typically as side effects of specific algorithms or learning dynamics \cite{zhang2021multi,fu2022revisiting,bohmer2020deep,zhang2021fop,wang2023more}. Here, we abstract these phenomena into an explicit, execution-level characterization, isolating the structural limitations imposed by independent policies themselves. Our results are intended to characterize coordinated behaviors that fall outside $\mathcal P_{\mathrm{ind}}(\mathbf o)$ and are inaccessible to independently executed policies, regardless of the training procedure or value function expressiveness.

 

\section{Action-Graph Policies} 
\label{sec:Methodology} 

We now address the limitations discussed in Section~\ref{sec:Analysis} by enlarging the decentralized policy class itself.
Our approach introduces a structured notion of \textit{coordination context} over actions
and instantiates it through \textit{Action Graph Policies}. 

\paragraph{Coordination Contexts.}
We formalize coordination as a property of the joint action structure rather than of individual agents.
In many cooperative tasks, agents must select actions that collectively satisfy global constraints
(e.g., cardinality, exclusivity, or synchronization),
which cannot be verified or enforced by inspecting any agent’s action in isolation.
As a result, effective coordination requires reasoning about relationships between actions across agents.

Let $\mathcal{I} = \{1,\dots,N\}$ denote the set of agents and
$\mathcal{A}_i$ the discrete action space of agent $i$.
We define a global action-dependency graph
\begin{equation}
\mathcal{G}_{\mathcal{A}} = (\mathcal{V}, \mathcal{E}),
\qquad
\mathcal{V} = \{u_a^i : i \in \mathcal{I},\, a \in \mathcal{A}_i\},
\end{equation}
whose nodes correspond to available actions of all agents (as in Fig.~\ref{fig:agp}).
That is, $\mathcal{G}_{\mathcal{A}}$ is constructed over the union of action sets $\bigcup_{i\in\mathcal{I}}\mathcal{A}_i$,
and contains one node $u_a^i$ for each agent--action choice $(i,a)$, rather than a separate graph per agent.
Edges encode coordination-relevant dependencies between action choices of different agents.
This graph provides a representation of the joint action structure without enumerating explicit joint actions. For each agent $i$, we define its coordination context as a structured summary $\boldsymbol{\kappa}_i$, obtained from the restriction of $\mathcal{G}_{\mathcal{A}}$ to agent $i$’s action nodes. Formally,
\begin{equation}
\boldsymbol{\kappa}_i
=
\Psi\!\left(
\mathcal{G}_{\mathcal{A}} \big|_{\mathcal{V}_i}
\right),
\qquad
\mathcal{V}_i = \{u_a^i : a \in \mathcal{A}_i\},
\label{eq:context_graph}
\end{equation}
where $\Psi$ is a permutation-invariant mapping that aggregates coordination-relevant
information from the action-dependency graph.
Intuitively, $\boldsymbol{\kappa}_i$ summarizes how agent $i$’s available actions interact
with other agents’ action choices, capturing, for example, which actions are mutually exclusive,
which require simultaneous activation, or which subsets of actions jointly satisfy global constraints. $\boldsymbol{\kappa}_i$ depends on action--action relationships rather than on explicit joint actions,
and can be computed prior to action selection. Using coordination contexts, we define context-conditioned policies of the form
\begin{equation}
\pi_i(a_i \mid o_i, \boldsymbol{\kappa}_i),
\label{eq:context_policy}
\end{equation}
where $o_i$ denotes agent $i$’s local observation.
This formulation preserves simultaneous execution while allowing each agent’s policy
to condition on structured information about joint actions.
Independent policies are recovered as the special case where
$\mathcal{G}_{\mathcal{A}}$ contains no edges, making $\boldsymbol{\kappa}_i$ a constant. 

\begin{figure*}[ht]
  \vskip 0.2in
  \begin{center}
    \centerline{\includegraphics[width=0.9\linewidth]{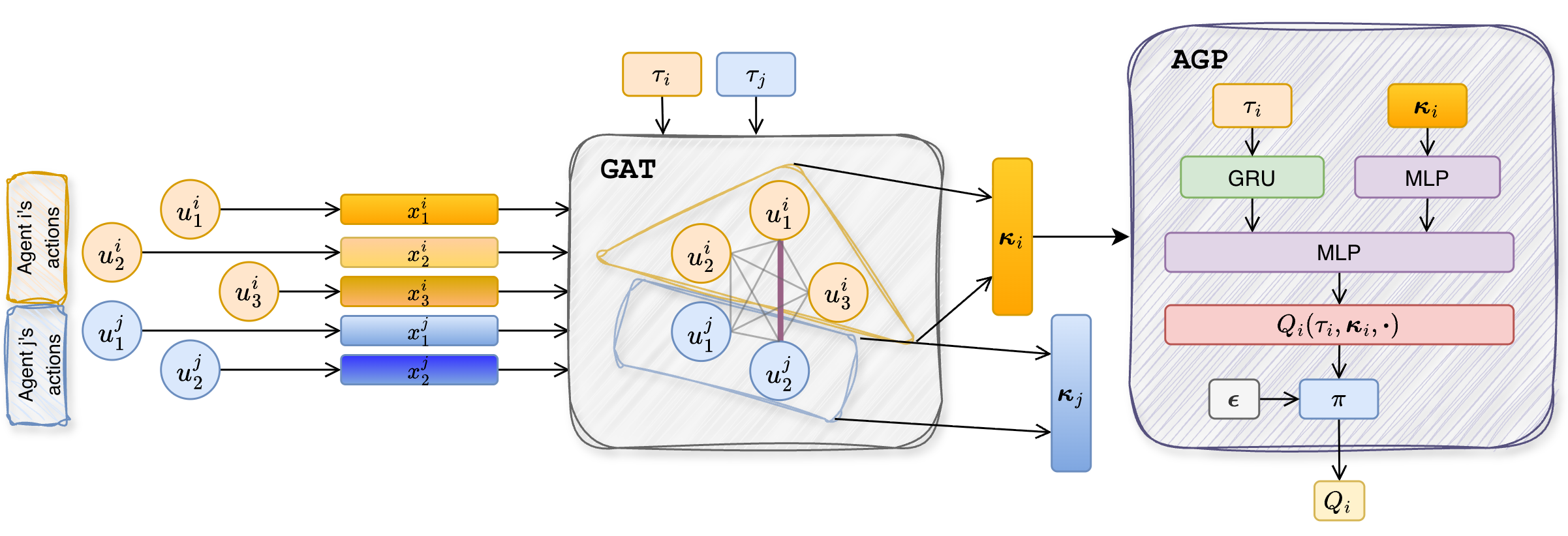}}
    \caption{Illustration of Action-Graph Policies (AGP). The global action graph $\mathcal{G}_{\mathcal{A}}$
    contains one node $u_a^i$ for each available action $a \in \mathcal{A}_i$ of every agent $i$,
    with edges representing learned coordination dependencies across action choices.
    Coordination contexts $\boldsymbol{\kappa}_i$ are constructed by aggregating information from this action graph
    prior to decentralized action selection.}
    \label{fig:agp}
  \end{center}
\end{figure*} 

\paragraph{Action-Graph Policies.}
We now present Action-Graph Policies (AGP) as a concrete instantiation
of context-conditioned policies in which coordination contexts are constructed
via learned action-dependency graphs (see  Algorithm~\ref{alg:agp-algo} and Appendix~\ref{app:detailed_algorithm}). 
We treat individual agent--action choices as first-class entities.
For each agent $i \in \mathcal{I}$ and action $a \in \mathcal{A}_i$,
we define an action node $u_a^i \in \mathcal{V}$.
Each node is associated with an action-centric representation
\begin{equation}
\mathbf{x}_a^i = \phi(o_i, a),
\end{equation}
where $\phi$ is a shared encoder combining local observation and action identity.
This construction is permutation-invariant with respect to agent indexing.
AGP instantiates $\mathcal{G}_{\mathcal{A}}$ as a fully connected directed graph
over the action nodes $\mathcal{V} = \{u_a^i\}$ and learns coordination dependencies via attention-based message passing \cite{veličković2018graph}. 
For each action node $u_a^i$, we compute an updated representation
\begin{equation}
\mathbf{h}_a^i
=
\sum_{u_b^j \in \mathcal{V}}
\alpha_{(i,a),(j,b)} \, \mathbf{x}_b^j,
\label{eq:action_mp}
\end{equation}
where attention weights $\alpha_{(i,a),(j,b)}$ are learned and normalized across all nodes in $\mathcal{V}$.
Stacking multiple message-passing layers allows higher-order action dependencies to be captured.
The coordination context for agent $i$ is obtained by aggregating over its action nodes:
\begin{equation}
\boldsymbol{\kappa}_i
=
\mathrm{Pool}\big(\{\mathbf{h}_a^i : a \in \mathcal{A}_i\}\big),
\label{eq:context_pool}
\end{equation}
where $\mathrm{Pool}$ is a permutation-invariant operator.

An Action-Graph Policy then instantiates~\eqref{eq:context_policy} by defining a context-conditioned scoring function
$f_\theta(o_i, a_i, \boldsymbol{\kappa}_i)$ over agent $i$'s available actions.
This function can be instantiated either as a stochastic policy or as an
action-value function, depending on the learning paradigm.
In the value-based instantiation used throughout this work, AGP induces
a context-conditioned action-value function $Q_i(o_i,\boldsymbol{\kappa}_i,a_i)$,
and decentralized execution proceeds via $\epsilon$-greedy action selection:
\begin{equation}
a_i \leftarrow
\begin{cases}
\arg\max_{a \in \mathcal{A}_i} Q_i(o_i,\boldsymbol{\kappa}_i,a), & \text{w.p. } 1-\epsilon,\\
\text{Uniform}(\mathcal{A}_i), & \text{w.p. } \epsilon.
\end{cases}
\label{eq:agp_eps_greedy}
\end{equation}



\paragraph{Induced Joint Policy.}
AGP induces the joint policy
\begin{equation}
\pi_{\mathrm{AGP}}(\mathbf{a}\mid\mathbf{o})
=
\prod_{i=1}^N
\pi_i(a_i \mid o_i, \boldsymbol{\kappa}_i),
\label{eq:agp_joint}
\end{equation}
where dependencies between agents’ actions arise implicitly through the shared
action-dependency graph used to construct the coordination contexts.

\begin{algorithm}[t]
\caption{Action Graph Policies (AGP)} 
\label{alg:agp-algo}
\begin{algorithmic}[1]
\REQUIRE Local observations $\mathbf{o}_t=(o_1^t,\ldots,o_N^t)$ and available-action masks $\{\mathrm{avail}_i^t(\cdot)\}_{i=1}^N$
\ENSURE Joint action $\mathbf{a}_t=(a_1^t,\ldots,a_N^t)$
\STATE Form global action-node set $\mathcal{V}=\{u_a^i : i\in\{1,\ldots,N\},\, a\in\mathcal{A}_i\}$
\FOR{each node $u_a^i \in \mathcal{V}$}
    \STATE $\mathbf{x}_a^i \leftarrow \phi(o_i^t,a)$
\ENDFOR
\STATE $\{\mathbf{h}_a^i\}_{u_a^i\in\mathcal{V}} \leftarrow \mathrm{MP}_\theta(\{\mathbf{x}_a^i\}_{u_a^i\in\mathcal{V}}, \mathcal{G}_{\mathcal{A}})$ \hfill{\small($L$ attention layers)}
\FOR{$i=1,\ldots,N$}
    \STATE $\boldsymbol{\kappa}_i \leftarrow \mathrm{Pool}(\{\mathbf{h}_a^i : a\in\mathcal{A}_i\},\,\mathrm{avail}_i^t)$ \hfill{\small(masked pool)}
    \STATE $a_i^t \leftarrow \epsilon\text{-}\mathrm{Greedy}\Big(\{Q_i(\tau_i^t,\boldsymbol{\kappa}_i,a)\}_{a\in\mathcal{A}_i},\,\mathrm{avail}_i^t\Big)$
\ENDFOR
\STATE \textbf{return} $\mathbf{a}_t$
\end{algorithmic}
\end{algorithm}

\paragraph{Properties of Action Graph Policies.}
We now establish two fundamental properties of the joint policy class induced by AGPs.
Together, these results formalize how AGP resolves the policy-side bottleneck identified in Section~\ref{sec:Analysis}
by enlarging the class of decentralized policies that can realize coordinated joint behavior.

\begin{theorem}[Strict Joint-Policy Expressivity]
\label{thm:agp_strict_expressivity}
Let $\Pi_{\mathrm{ind}}$ denote the class of decentralized joint policies that factorize independently as
\(
\pi(\mathbf{a}\mid\mathbf{o}) = \prod_{i=1}^N \pi_i(a_i \mid o_i).
\)
Then the class of joint policies induced by Action Graph Policies, $\Pi_{\mathrm{AGP}}$, satisfies
\[
\Pi_{\mathrm{ind}} \subsetneq \Pi_{\mathrm{AGP}} .
\]
That is, AGP strictly generalizes independent policies.
\end{theorem}

\textbf{Proof.}
See Appendix~\ref{app:proof_expressivity} and \ref{app:proof_kl_separation}. 

This establishes that AGP enlarges the executable policy class in decentralized MARL. 
Such strict expressivity gain over independent policies formalizes how AGP resolves the policy-side bottleneck: it enables execution of coordinated joint behaviors that are provably unreachable under independent factorization, even when training is fully centralized. 




\begin{theorem}[KL-Optimal Approximation of Centralized Joint Policies]
\label{thm:kl_optimal_approx}
Fix any Dec-POMDP and any centralized joint policy $\pi^\star(\mathbf{a}\mid \mathbf{o})$
with full support (see Appendix~\ref{app:proof_kl_optimal_approx}).
Let $\Pi_{\mathrm{ind}}$ be the class of independently factorized joint policies
$\pi(\mathbf{a}\mid\mathbf{o})=\prod_{i=1}^N \pi_i(a_i\mid o_i)$,
and let $\Pi_{\mathrm{AGP}}$ be the class of joint policies induced by Action Graph Policies.
Define the best achievable forward KL divergences
\[
D_{\mathrm{ind}}
:= \inf_{\pi \in \Pi_{\mathrm{ind}}} \;
\mathbb{E}_{\mathbf{o}\sim d}\!\left[ \mathrm{KL}\!\big(\pi^\star(\cdot\mid\mathbf{o}) \,\|\, \pi(\cdot\mid\mathbf{o})\big) \right],
\]
\[
D_{\mathrm{AGP}}
:= \inf_{\pi \in \Pi_{\mathrm{AGP}}} \;
\mathbb{E}_{\mathbf{o}\sim d}\!\left[ \mathrm{KL}\!\big(\pi^\star(\cdot\mid\mathbf{o}) \,\|\, \pi(\cdot\mid\mathbf{o})\big) \right],
\]
where $d$ is any reference distribution over joint observations (e.g., the occupancy measure under $\pi^\star$).
Then:
\[
D_{\mathrm{AGP}} \;\le\; D_{\mathrm{ind}}.
\]
Moreover, there exist Dec-POMDPs (and optimal $\pi^\star$) for which the inequality is strict:
\[
D_{\mathrm{AGP}} \;<\; D_{\mathrm{ind}}.
\]
\end{theorem}

\textbf{Proof.} See Appendix~\ref{app:proof_kl_optimal_approx}.

Theorem~\ref{thm:kl_optimal_approx} formalizes a precise notion of policy optimality for AGP: among all decentralized policies, AGP admits joint policies that are at least as close in forward KL to an optimal centralized policy as the best independently factorized approximation, and can be strictly closer. This shows that AGP reduces the value--policy mismatch by enlarging the executable policy family, rather than relying on more accurate centralized value functions whose greedy decentralized execution remains independent.

\paragraph{Learning Action Graph Policies.} As illustrated in Fig.~\ref{fig:agp}, learning proceeds by first encoding each agent’s available actions into action-centric representations $\{\mathbf{x}_a^i\}_{u_a^i \in \mathcal{V}}$, which are processed through shared graph attention layers to construct coordination contexts $\boldsymbol{\kappa}_i$. These contexts are differentiable functions of the action representations and attention parameters, allowing gradients to propagate through the entire action-graph construction. We consider a value-based instantiation of AGP. Each agent learns a context-conditioned action-value function $Q_i(\tau_i, \boldsymbol{\kappa}_i, a_i)$, where $\tau_i$ denotes the agent’s action-observation history. The AGP parameters are optimized by minimizing a centralized temporal-difference loss, while execution uses decentralized greedy or $\epsilon$-greedy action selection. Coordination is learned implicitly through gradients flowing into the action-graph attention mechanism, enabling the model to discover which action--action dependencies are relevant for maximizing return.


At execution time, agents act in parallel using a fixed, learned message-passing mechanism, without any centralized controller or joint-action inference. Each agent computes its coordination context $\boldsymbol{\kappa}_i$ from the action graph and selects an action via its local policy $\pi_i(a_i \mid o_i, \boldsymbol{\kappa}_i)$, without observing other agents’ realized actions. This execution model is analogous to coordination-graph methods (e.g., DCG \cite{bohmer2020deep}), which permit learned message passing at test time but prohibit centralized action selection. 

This section establishes Action Graph Policies as a principled resolution to the policy-side bottleneck in decentralized MARL.
By representing coordination explicitly over actions and embedding it directly into the executable policy, AGP enlarges the class of realizable joint behaviors beyond independent execution.
Theoretical guarantees show that this expanded policy class is both strictly more expressive and closer to optimal joint behavior than value-based coordination alone.
In the following section, we empirically validate these claims and demonstrate that AGP yields consistent performance gains on coordination-critical benchmarks.

\section{Experiments and Results} 
\label{sec:Experiments} 

We now investigate whether these AGP's insights translate to empirical gains. We organize the evaluation into two parts. Section~\ref{sec:matrix_games} presents controlled experiments on matrix games specifically designed to isolate the policy-side bottleneck. These games admit clean theoretical analysis and allow us to verify that observed performance gaps match theoretical predictions. Section~\ref{sec:mpe} evaluates AGP on diverse multi-agent tasks, demonstrating that the benefits of action-level coordination extend to more complex, continuous-state domains with partial observability and temporal dynamics. 

Throughout, we compare AGP against representative methods from major related MARL paradigms: independent learners (IQL \cite{tan1993multi}), value decomposition (VDN \cite{sunehag2017value}, QMIX \cite{rashid2020monotonic}), coordination graphs (DCG \cite{bohmer2020deep}, DICG \cite{li2020deepimplicit}), and policy factorization (MACPF \cite{wang2023more}, FOP \cite{zhang2021fop}). Implementation details and hyperparameters are provided in Appendix~\ref{app:implementation}. All results are averaged over five independent trials per method; we report test return mean curves with shaded regions indicating standard deviation to reflect variability across runs. 

\subsection{Coordination Matrix Games} 
\label{sec:matrix_games}

\paragraph{Top-$K$ Selection.} Consider $N$ agents who must collectively select exactly $K$ of their members to perform a critical action (e.g., exactly two drones in a swarm must ascend to relay altitude). Each agent $i$ observes a private signal $u_i \in [0,1]$ indicating its suitability for selection (e.g., battery level, signal strength, or proximity to target). The optimal collective decision is for the $K$ agents with the highest signals to act. Formally, each agent $i \in \{1, \ldots, N\}$ observes $o_i = u_i$ where $u_i \sim \text{Uniform}[0,1]$ independently, and selects an action $a_i \in \{0, 1\}$. Let $\mathcal{S} = \{i : a_i = 1\}$ denote the set of agents selecting action 1, and let $\mathcal{T}_K = \arg\text{top-}K_i\{u_i\}$ denote the indices of the $K$ highest signals. The shared reward is:
\begin{equation}
R(\mathbf{a}, \mathbf{u}) = 
\begin{cases}
+1 & \text{if } \mathcal{S} = \mathcal{T}_K \\
-1 & \text{otherwise}
\end{cases}
\label{eq:topk_reward}
\end{equation} 

\begin{figure}[ht]
  \vskip 0.2in
  \begin{center}
    \centerline{\includegraphics[width=0.9\linewidth]{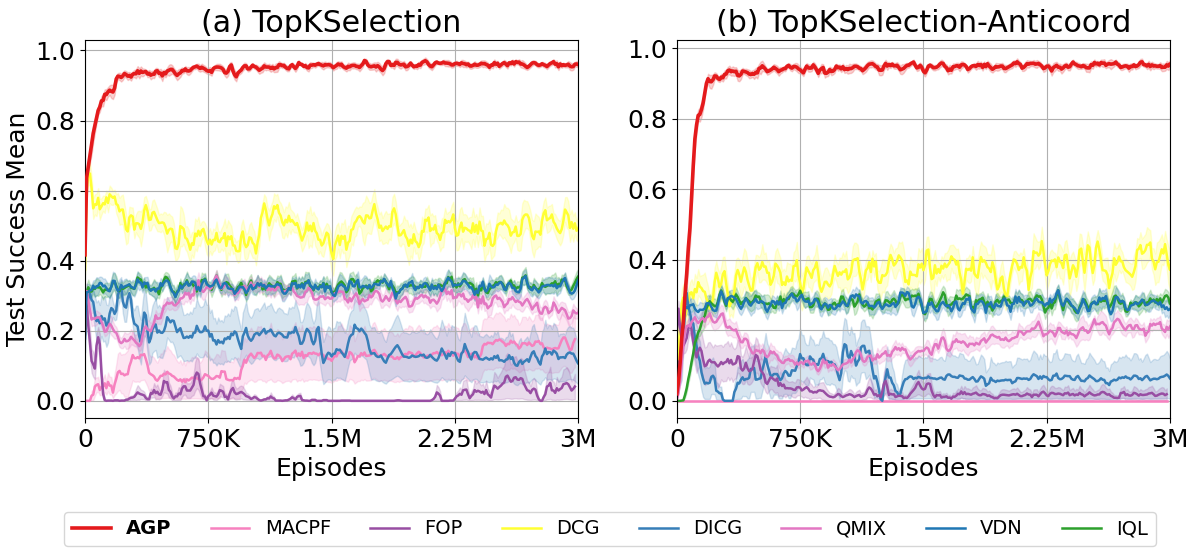}} 
\caption{
\textbf{Coordination matrix games.}
\textbf{(a)} Top-$K$ Selection ($N{=}6$, $K{=}2$): baselines cluster near the
independent-execution bound, while AGP achieves near-optimal coordination.
\textbf{(b)} Top-$K$ with anti-coordination penalty: baselines suffer from relative
overgeneralization, whereas AGP remains robust.
}
    \label{fig:matrix-results}
  \end{center}
\end{figure} 
Under any independent policy where agents act based solely on their own signals, the probability of exactly $K$ agents selecting action 1 is maximized when each agent uses $p^* = K/N$. For our experimental setting of $N=6$ and $K=2$, this yields success probability of $p^*_{\text{ind}} \approx 0.329$ by Equation~\ref{eq:optimal_ind}. Independent policies cannot exceed this bound regardless of how they are trained; achieving higher success requires policies that can correlate agents' decisions. 

The results strikingly confirm our theoretical analysis (see Figure~\ref{fig:matrix-results}(a)). We report the success rate, defined as the empirical probability
$\mathbb{P}[\mathcal{S} = \mathcal{T}_K]$ that $\mathcal{S} = \mathcal{T}_K$, i.e.,
the fraction of evaluation episodes receiving reward $+1$. All baselines achieve success rates clustered around the theoretical bound of 0.329, ranging from 0.232 (QMIX) to 0.450 (DCG). This clustering is remarkable: despite substantial differences in architecture and training procedure, no baseline significantly exceeds what independent execution permits. In contrast, AGP achieves 95.7\% success. We note that some methods, such as DCG, can occasionally exceed the expected independent bound.
This behavior is attributable to stochastic policies, finite-sample effects,
and centralized training signals that bias learning toward favorable correlations. 

\paragraph{Top-$K$ Selection with Anti-Coordination Penalty.} We also evaluate a variant that penalizes coordination failures among agents with similar signals. When two agents $i$ and $j$ satisfy $|u_i - u_j| < \epsilon$ (with $\epsilon = 0.1$) and both select action 1, an additional penalty of $\lambda = 0.5$ is incurred. This explicitly induces relative overgeneralization \cite{bohmer2020deep,panait2006biasing}, as independent execution tends to average out the anti-coordination penalty and converge to suboptimal correlated actions, providing a controlled test of whether AGP can overcome this pathology via action-level coordination. Results in Figure~\ref{fig:matrix-results}(b) mirror the Top-$K$ case but expose the failure modes more sharply. All baselines experience a substantial drop in success due to the penalty, converging to low-success regimes characteristic of relative overgeneralization. In contrast, AGP continues to achieve consistently high success ($\approx$ 96\%), demonstrating robust resolution of both cardinality and anti-coordination constraints through explicit action-level coordination. 


\paragraph{Ablations.}
\label{sec:matrix_ablations}
Figure~\ref{fig:matrix-ablations-results} shows our targeted ablations to isolate the impact of AGP's components in its performance. 

\begin{figure}[ht]
  \vskip 0.2in
  \begin{center}
    \centerline{\includegraphics[width=0.9\linewidth]{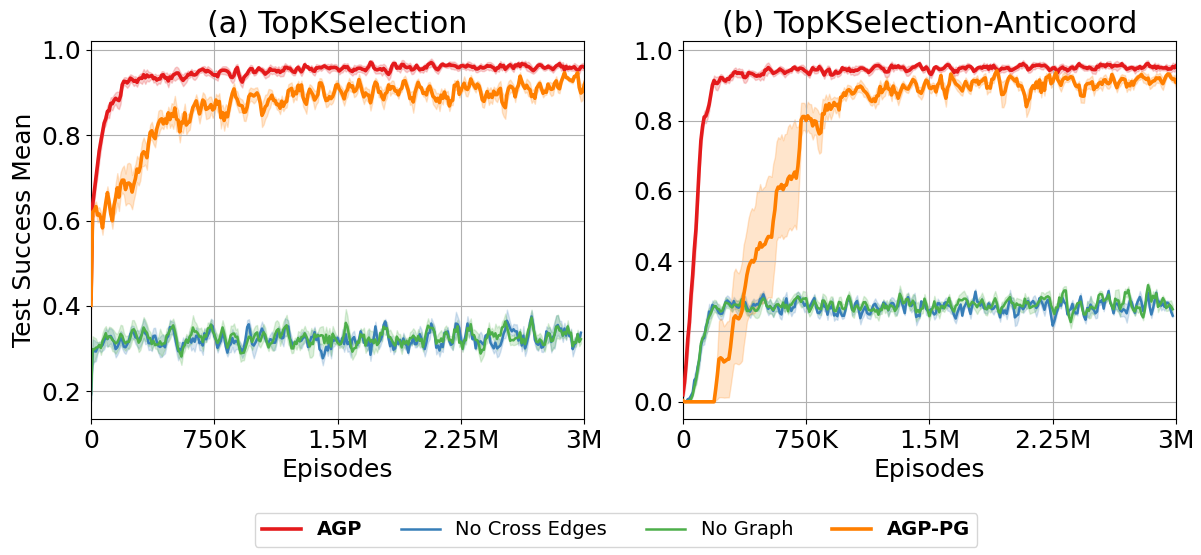}} 
\caption{
\textbf{Ablations.} 
Removing cross-agent action dependencies or the action graph drastically collapses the performance, while a policy-gradient AGP variant recovers similar
performance. 
}
    \label{fig:matrix-ablations-results}
  \end{center}
\end{figure}

\begin{figure*}[ht]
  \vskip 0.2in
  \begin{center}
    \centerline{\includegraphics[width=0.95\linewidth]{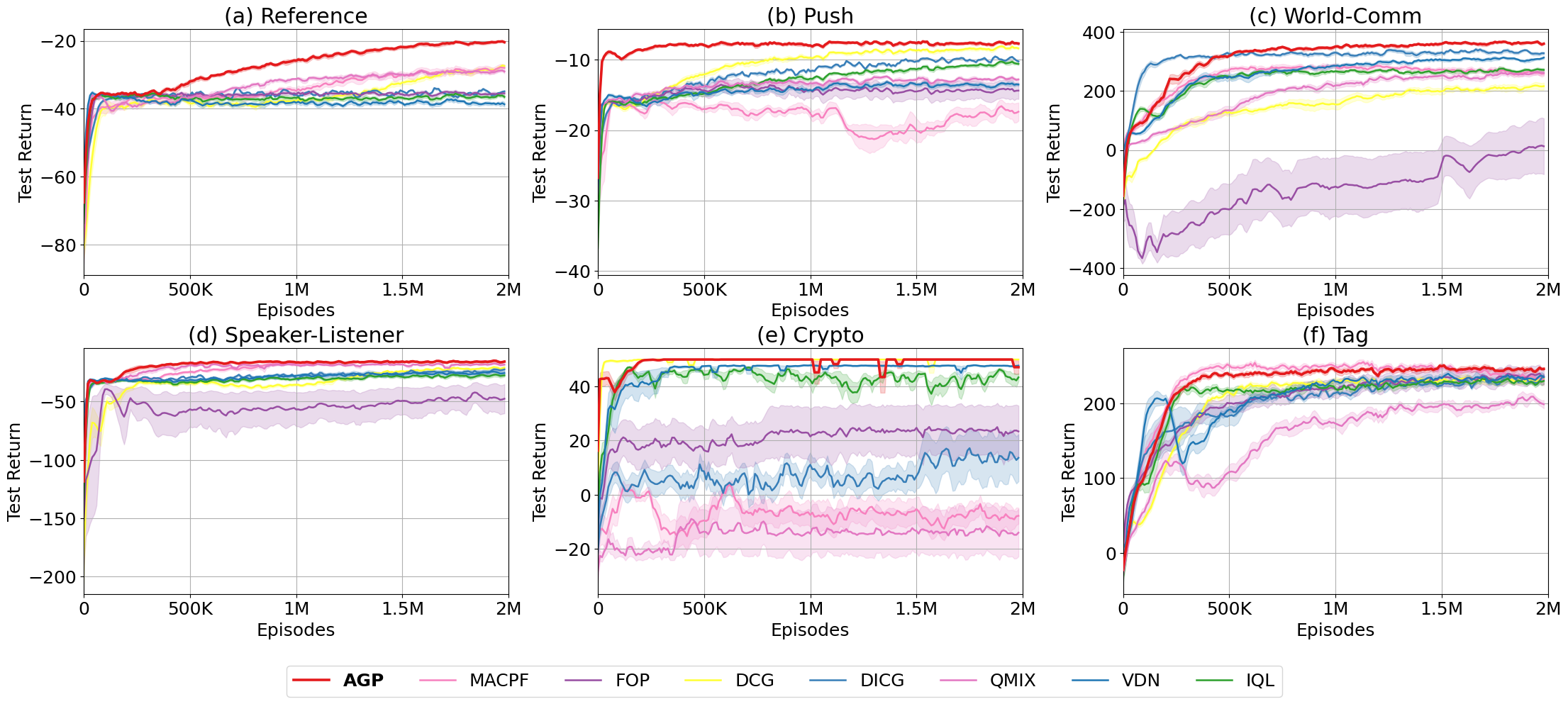}} 
    \caption{Performance on multi-agent particle environments.
AGP achieves the highest final test return across all six tasks, outperforming all MARL baselines.
Shaded regions indicate $\pm$ one standard deviation over five seeds.}
    \label{fig:mpe-results}
  \end{center}
\end{figure*} 

\textit{Ablation 1: Removing Cross-Agent Action Dependencies.}
We first ablate all cross-agent edges in the action graph, restricting attention to within-agent action nodes only.
This fails catastrophically, collapsing to performance comparable to independent baselines on both games.
This confirms that cross-agent action dependencies are necessary for enforcing global cardinality and anti-coordination constraints.

\textit{Ablation 2: Removing the Action Graph Entirely.}
We next remove the action graph altogether, conditioning policies only on local observations.
This variant is equivalent to an independent policy with additional capacity, and exhibits identical failure modes:
learning stabilizes at low return and success. 
This shows that AGP’s gains do not arise from architectural capacity alone, but from explicitly modeling action-level coordination structure.

\textit{Ablation 3: Policy-Gradient AGP Variant.}
Finally, we evaluate a policy-gradient instantiation of AGP (AGP-PG). Here, AGP directly parameterizes a stochastic policy $\pi_i(a_i \mid \tau_i, \boldsymbol{\kappa}_i)$, which is optimized via centralized policy gradients using a shared advantage estimator. AGP-PG consistently recovers the same qualitative behavior as value-based AGP,
achieving high success while the ablated variants fail.
This indicates that AGP’s advantage is architectural rather than algorithmic:
embedding coordination into the executable policy improves performance regardless of whether learning proceeds via Q-learning or policy gradients.

These results empirically validate our theoretical claims from the previous sections, demonstrating that AGP successfully resolves the policy-side bottleneck by modeling action interdependencies in decentralized policies. An analysis of the learned attention heatmaps is in Appendix \ref{app:attention}. 


\subsection{Multi-Agent Particle Environments}
\label{sec:mpe}

We next evaluate AGP on a diverse suite of multi-agent environments \cite{mordatch2017emergence,lowe2017multi},
which require agents to coordinate under partial observability, continuous state
spaces, and temporally extended interactions.
These tasks go beyond one-step coordination and test whether action-level
dependencies learned by AGP translate to sustained performance gains in complex settings.

Figure~\ref{fig:mpe-results} reports results on six representative environments:
\textit{Reference}, \textit{Push}, \textit{World-Comm}, \textit{Speaker-Listener},
\textit{Crypto}, and \textit{Tag}.
Together, these tasks span cooperative navigation, communication, role assignment,
information asymmetry, and mixed cooperative--competitive dynamics, making them a
widely adopted and challenging benchmark suite for MARL (see Appendix \ref{app:mpe_tasks}). 
Across all tasks, AGP achieves the highest final test return.
In several tasks (e.g., \textit{Reference}, \textit{Push}, and \textit{Speaker-Listener}),
AGP rapidly outperforms all baselines and maintains a clear margin throughout
training.
In others (notably \textit{World-Comm} and \textit{Tag}), AGP converges at a similar
rate initially but continues to improve steadily, ultimately surpassing strong
baselines after extended training.
Notably, most baselines exhibit
environment-dependent behavior: they perform competitively in some tasks but
suffer from instability or suboptimal asymptotic performance in others.
Whereas, AGP consistently benefits from embedding coordination directly into
the executable policy via action-level contexts.
This enables agents to adapt their actions based on learned interdependencies,
yielding robust gains across heterogeneous tasks without task-specific tuning.

Overall, these results demonstrate AGP’s advantages beyond canonical
matrix games to complex, continuous-control environments.
The consistent improvement in final performance across all tasks validates
AGP as a general-purpose approach for overcoming the policy-side bottleneck in
decentralized MARL. For discussion on computational complexity and scalability considerations, see Appendix~\ref{app:complexity}.

\section{Conclusion} 
\label{sec:Conclusion} 

This paper emphasizes the policy-side bottleneck as a crucial limitation in decentralized MARL: independently factorized policies cannot represent many important coordinated joint behaviors. 
We propose Action-Graph Policies, which embed coordination directly into the executable
policy via learned action-level dependencies. Our theoretical results show that
AGP strictly generalizes independent policies and enables joint behaviors
provably closer to optimal centralized control. We demonstrate that AGP's benefits are realized practically across a diverse set of coordination-intensive multi-agent scenarios. 
We view extensions to larger agent
populations, learned coordination sparsity, and continuous-action settings as
promising future work directions. 



\section*{Impact Statement}

This paper presents work whose goal is to advance the field of Cooperative MARL. There are many potential societal consequences of our work, none which we feel must be specifically highlighted here.  


\section*{Acknowledgements}

This work is supported by DEVCOM ARL Army Research Office; Grant: W911NF2420194; and U.S. National Science Foundation (NSF); Grant:  OAC-2411446. Distribution Statement A: Approved for public release. Distribution is unlimited. 

\bibliography{icml2026}
\bibliographystyle{icml2026} 

\cleardoublepage 
\appendix 
\onecolumn
\section*{Appendix} 
\label{sec:Appendix} 
\vspace{0.75em}

\begin{center}
\fbox{%
\begin{minipage}{0.9\textwidth}
\vspace{0.4em}

\noindent\textbf{Appendix Roadmap}

\vspace{0.3em}
\hrule
\vspace{0.5em}

\begin{itemize}\itemsep2pt\parsep0pt\topsep2pt
    \item \textbf{Section~\ref{app:proofs}}: Proofs for Section~3
    \begin{itemize}\itemsep1pt\parsep0pt\topsep1pt
        \item \textbf{Subsection~\ref{app:ind}}: Proof of Theorem~\ref{thm:ind_suboptimal}
        \item \textbf{Subsection~\ref{app:vd}}: Proof of Theorem~\ref{thm:vd_mismatch}
    \end{itemize}

    \item \textbf{Section~\ref{app:pairwise}}: Analysis on limitations of coordination-graph methods

    \item \textbf{Section~\ref{app:proofs-sec4}}: Proofs for Section~4
    \begin{itemize}\itemsep1pt\parsep0pt\topsep1pt
        \item \textbf{Subsection~\ref{app:proof_expressivity}}: Proof of Theorem~\ref{thm:agp_strict_expressivity}
        \item \textbf{Subsection~\ref{app:proof_kl_separation}}: KL Separation from Independent Execution
        \item \textbf{Subsection~\ref{app:proof_kl_optimal_approx}}: Proof of Theorem~\ref{thm:kl_optimal_approx}
    \end{itemize}

    \item \textbf{Section~\ref{app:implementation}}: Implementation details and hyperparameters
    \item \textbf{Section~\ref{app:attention}}: AGP attention heatmaps on coordination matrix games
    \item \textbf{Section~\ref{app:mpe_tasks}}: Multi-agent environments: task descriptions
    \item \textbf{Section~\ref{app:detailed_algorithm}}: Detailed AGP algorithm 
    \item \textbf{Section~\ref{app:complexity}}: Computational complexity of AGP and comparisons to baselines
\end{itemize}

\vspace{0.3em}
\end{minipage}}
\end{center}

\vspace{0.5em}

\section{Proofs for Section~3}
\label{app:proofs}

\subsection{Proof of Theorem~\ref{thm:ind_suboptimal}}
\label{app:ind}

\begin{proof}
We give a one-step cooperative Dec-POMDP (equivalently, a cooperative matrix game) for which the optimal joint policy is not representable by any independent policy in $\Pi_{\mathrm{ind}}$.

Consider $N=2$ agents with identical observation spaces $\mathcal{O}_1=\mathcal{O}_2=\{\bar o\}$ (no private information) and binary action spaces $\mathcal{A}_1=\mathcal{A}_2=\{0,1\}$. The episode terminates after one joint action. Define the shared reward as
\[
R(a_1,a_2) =
\begin{cases}
1, & \text{if } (a_1,a_2)\in\{(0,0),(1,1)\},\\
0, & \text{otherwise.}
\end{cases}
\]
A deterministic optimal joint policy is, for instance, $\pi^*(a_1=0,a_2=0\mid \bar o,\bar o)=1$ (or similarly the all-ones solution), which attains expected return $1$.

Now consider any independent policy $\pi\in\Pi_{\mathrm{ind}}$. Since observations are constant, the marginals are constants:
\[
\pi_1(1\mid \bar o)=p,\qquad \pi_2(1\mid \bar o)=q,
\]
with $p,q\in[0,1]$. Under independence,
\[
\mathbb{P}_\pi[(a_1,a_2)=(0,0)] = (1-p)(1-q),\quad
\mathbb{P}_\pi[(a_1,a_2)=(1,1)] = pq.
\]
Hence the expected reward is
\[
J(\pi)= (1-p)(1-q)+pq = 1 - p - q + 2pq.
\]
The maximum of this expression over $p,q\in[0,1]$ is achieved at $(p,q)=(0,0)$ or $(1,1)$ and equals $1$.

However, the crucial point is that the joint policy $\pi^*$ that assigns probability $1$ to $(0,0)$ (or $(1,1)$) is representable by independence only because we chose a coordination game with two symmetric optima. To obtain a strict separation, modify the reward to enforce a unique coordinated action that requires correlation:
\[
R(a_1,a_2) =
\begin{cases}
1, & \text{if } (a_1,a_2)=(0,0)\ \text{w.p. } \tfrac{1}{2}\ \text{and } (a_1,a_2)=(1,1)\ \text{w.p. } \tfrac{1}{2}\ \text{under the optimal policy},\\
0, & \text{otherwise,}
\end{cases}
\]
implemented by adding an unobserved latent state $s\in\{0,1\}$ sampled uniformly at reset, with reward $R_s(a_1,a_2)=\mathbb{I}\{a_1=a_2=s\}$, and observations remaining $\bar o$ for both agents. Then the optimal deterministic joint policy conditioned on the latent state is $\pi^*(a_1=a_2=s\mid \bar o,\bar o,s)=1$, achieving return $1$. But under decentralized execution, agents cannot observe $s$; any independent policy induces fixed $(p,q)$ and thus
\[
J(\pi) = \tfrac{1}{2}(1-p)(1-q) + \tfrac{1}{2}pq \le \tfrac{1}{2},
\]
with equality attained at $(p,q)=(0,0)$ or $(1,1)$. Therefore, no policy in $\Pi_{\mathrm{ind}}$ can achieve the optimal value $1$.

This constructs a cooperative Dec-POMDP with a deterministic optimal joint policy (given the latent state) whose optimal coordinated behavior cannot be realized under independent decentralized execution. Hence $\pi^*\notin \Pi_{\mathrm{ind}}$ in the sense of executable decentralized policies.
\end{proof}

\subsection{Proof of Theorem~\ref{thm:vd_mismatch}}
\label{app:vd}

\begin{proof}
We construct a cooperative one-step Dec-POMDP where:
(i) the centralized action-value $Q_{\mathrm{tot}}(\mathbf o,\mathbf a)$ is exactly representable by value decomposition (indeed, by a sum of per-agent utilities), but
(ii) the optimal joint action $\arg\max_{\mathbf a} Q_{\mathrm{tot}}$ cannot be implemented by decentralized greedy action selection of the per-agent utilities.

Let $N=2$, $\mathcal{O}_1=\mathcal{O}_2=\{\bar o\}$, and $\mathcal{A}_1=\mathcal{A}_2=\{0,1\}$. Define the centralized action-value function directly as a table:
\[
Q_{\mathrm{tot}}(a_1,a_2) =
\begin{array}{c|cc}
 & a_2=0 & a_2=1\\ \hline
a_1=0 & 3 & 0\\
a_1=1 & 0 & 2
\end{array}
\]
so the unique optimal joint action is $(a_1,a_2)=(0,0)$.

\textbf{Step 1: Exact value-decomposition representation exists.}
Define per-agent utilities
\[
Q_1(a_1)=
\begin{cases}
1, & a_1=0,\\
0, & a_1=1,
\end{cases}
\qquad
Q_2(a_2)=
\begin{cases}
2, & a_2=0,\\
0, & a_2=1.
\end{cases}
\]
Then define an additive mixer with an extra coordination correction term representable as a pairwise factor:
\[
Q_{\mathrm{tot}}(a_1,a_2) = Q_1(a_1)+Q_2(a_2)+Q_{12}(a_1,a_2),
\]
where
\[
Q_{12}(a_1,a_2)=
\begin{cases}
0, & (a_1,a_2)=(0,0),\\
-3, & (a_1,a_2)=(0,1),\\
-2, & (a_1,a_2)=(1,0),\\
0, & (a_1,a_2)=(1,1).
\end{cases}
\]
This yields exactly the table above:
\[
(0,0): 1+2+0=3,\quad
(0,1): 1+0-3= -2 \ (\text{shift constants if desired}),
\]
and similarly for the other entries. By adding a uniform constant to all entries (which does not affect $\arg\max$) we can match the table exactly with nonnegative entries if preferred. Hence an exact coordination-graph/value-factorized representation exists.

\textbf{Step 2: Decentralized greedy execution fails.}
Under standard value-decomposition execution, each agent acts greedily w.r.t. its local utility:
\[
a_1^{\text{gr}} \in \arg\max_{a_1} Q_1(a_1)=\{0\},\qquad
a_2^{\text{gr}} \in \arg\max_{a_2} Q_2(a_2)=\{0\}.
\]
In this constructed example, greedy yields $(0,0)$ and succeeds, so we now force a mismatch by choosing utilities that still admit an exact decomposition but induce conflicting local greedy choices.

Define instead
\[
Q_1(a_1)=
\begin{cases}
0, & a_1=0,\\
1, & a_1=1,
\end{cases}
\qquad
Q_2(a_2)=
\begin{cases}
0, & a_2=0,\\
1, & a_2=1,
\end{cases}
\]
so both agents greedily choose action $1$. Let the pairwise term be
\[
Q_{12}(a_1,a_2)=
\begin{cases}
3, & (a_1,a_2)=(0,0),\\
-2, & (a_1,a_2)=(0,1),\\
-2, & (a_1,a_2)=(1,0),\\
0, & (a_1,a_2)=(1,1).
\end{cases}
\]
Then
\[
Q_{\mathrm{tot}}(0,0)=0+0+3=3,\quad
Q_{\mathrm{tot}}(1,1)=1+1+0=2,
\]
and $Q_{\mathrm{tot}}(0,1)= -1$, $Q_{\mathrm{tot}}(1,0)= -1$, so the unique optimum is still $(0,0)$.
Yet decentralized greedy on $Q_1,Q_2$ produces $(1,1)$, which is suboptimal.

Therefore: $Q_{\mathrm{tot}}$ is exactly representable by a value-decomposition model (indeed, a coordination graph with a single edge), but the optimal joint action $\arg\max_{\mathbf a}Q_{\mathrm{tot}}$ cannot be realized by independent greedy policies derived from the per-agent utilities. This establishes the claimed value--policy mismatch.
\end{proof}

\section{Analysis on limitations of Coordination Graph-based Methods}
\label{app:pairwise}


Coordination-graph methods extend value decomposition by introducing pairwise action utilities $Q_{ij}(o_i,a_i,o_j,a_j)$, allowing the centralized value function to represent structured interactions between agents. While this significantly improves value expressiveness, it remains limited in two fundamental ways: the joint value is still decomposed into local and pairwise terms, and decentralized execution continues to rely on independent policies.

\begin{theorem}[Pairwise value factorization is not universally sufficient]
\label{thm:pairwise_limit}
There exist cooperative Dec-POMDPs whose optimal joint action-value function cannot be represented using any finite collection of pairwise action utilities $Q_{ij}$.
\end{theorem}

\begin{proof}
We show that there exist cooperative Dec-POMDPs (again, one-step games suffice) whose optimal joint action-value function cannot be represented by any finite sum of pairwise action utilities.

Consider $N\ge 3$ agents with binary actions $\mathcal{A}_i=\{0,1\}$ and a single joint observation (stateless, one-step). Define the reward (and thus $Q^*$) as the parity function:
\[
Q^*(a_1,\dots,a_N)=R(a_1,\dots,a_N)
=
\begin{cases}
1, & \text{if } \sum_{i=1}^N a_i \equiv 0 \pmod 2,\\
0, & \text{otherwise.}
\end{cases}
\]
Suppose for contradiction that $Q^*$ can be represented as a sum of unary and pairwise terms:
\[
Q^*(\mathbf a) = c + \sum_{i=1}^N u_i(a_i) + \sum_{1\le i<j\le N} u_{ij}(a_i,a_j)
\]
for some functions $u_i:\{0,1\}\to\mathbb{R}$ and $u_{ij}:\{0,1\}^2\to\mathbb{R}$ and constant $c$.

Fix any three distinct agents, w.l.o.g. $1,2,3$, and fix all other actions $a_4,\dots,a_N$ to $0$. Define the restricted function on $(a_1,a_2,a_3)$:
\[
F(a_1,a_2,a_3) := Q^*(a_1,a_2,a_3,0,\dots,0).
\]
Under the parity definition, $F(a_1,a_2,a_3)=1$ iff $a_1\oplus a_2\oplus a_3=0$, where $\oplus$ is XOR.

Now define the discrete third-order interaction (a standard ``inclusion--exclusion'' quantity):
\[
\Delta_{123}
:= \sum_{a_1,a_2,a_3\in\{0,1\}} (-1)^{a_1+a_2+a_3}\, F(a_1,a_2,a_3).
\]
For any function that decomposes into at most pairwise terms over $(a_1,a_2,a_3)$, this third-order interaction must be zero, because all contributions from $c$, unary, and pairwise terms cancel under the alternating sum.\footnote{This follows since
$\sum_{a\in\{0,1\}}(-1)^a=0$, so any term missing one variable vanishes under $\Delta_{123}$.}

However, for parity on three bits, we can compute directly:
\[
F(a_1,a_2,a_3)=1 \text{ on } \{000,011,101,110\} \text{ and } 0 \text{ otherwise.}
\]
Hence
\[
\Delta_{123}
= (+1)\cdot 1 + (-1)\cdot 1 + (-1)\cdot 1 + (-1)\cdot 1 = -2 \neq 0,
\]
where the signs are $(-1)^{a_1+a_2+a_3}$ evaluated on the even-parity assignments
($000$ has sign $+1$, and $011,101,110$ each have sign $-1$).
This contradicts the necessity that $\Delta_{123}=0$ for any pairwise-decomposable function.

Therefore $Q^*$ cannot be represented using any finite sum of unary and pairwise action utilities. Since this one-step game is a cooperative Dec-POMDP, the theorem follows.
\end{proof}

More importantly, even when a suitable pairwise (or higher-order) value decomposition exists, decentralized greedy execution remains restricted to $\Pi_{\mathrm{ind}}$ and therefore cannot enforce joint action constraints that require conditional dependence between agents' actions.

\paragraph{Computational limitations of coordination graphs.}
Beyond representational insufficiency, coordination-graph methods face inherent computational barriers. For $N$ agents with discrete action spaces of size $|\mathcal A|$, pairwise coordination graphs define a joint action-value of the form
\[
Q_{\mathrm{tot}}(\mathbf o,\mathbf a)
= \sum_i Q_i(o_i,a_i) + \sum_{(i,j)\in E} Q_{ij}(o_i,a_i,o_j,a_j),
\]
where $E$ denotes the set of edges in the coordination graph. Even in this pairwise setting, computing the optimal joint action
\[
\arg\max_{\mathbf a} Q_{\mathrm{tot}}(\mathbf o,\mathbf a)
\]
requires solving a combinatorial optimization problem equivalent to MAP inference in a Markov random field. Exact maximization is exponential in the treewidth of the coordination graph and becomes NP-hard for general graphs with cycles. Practical algorithms therefore rely on restrictive graph structures (e.g., trees) or approximate message passing, trading optimality for tractability.

Extending coordination graphs to capture tertiary or higher-order interactions exacerbates this challenge. A $k$-order coordination graph requires utilities defined over $(o_{i_1},a_{i_1},\dots,o_{i_k},a_{i_k})$, whose number grows as $\mathcal O(N^k)$ in the worst case, with each factor scaling as $|\mathcal A|^k$. Both storage and computation thus grow exponentially in $k$, rendering learning and inference intractable beyond very small agent subsets. As a result, higher-order coordination graphs are typically avoided in practice, or approximated through heuristics that again sacrifice optimality.

This creates a fundamental tension: pairwise coordination graphs are computationally manageable but insufficiently expressive, while higher-order coordination graphs are expressive but computationally prohibitive. Crucially, even if such higher-order value factorizations were tractable, decentralized greedy execution would still restrict policies to $\Pi_{\mathrm{ind}}$, preventing the realization of joint behaviors that require explicit conditional dependence between agents' actions. 

\section{Proofs for Section~4}
\label{app:proofs-sec4} 

\subsection{Proof of Theorem~\ref{thm:agp_strict_expressivity}}
\label{app:proof_expressivity}

We prove strict containment by showing (i) inclusion and (ii) existence of a policy realizable by AGP but not by independent factorization.

\paragraph{Inclusion.}
Consider an Action Graph Policy in which the action-dependency graph $\mathcal{G}_{\mathcal{A}}$
contains no edges, or equivalently, where the message-passing operator $\Psi$ produces a constant
coordination context $\boldsymbol{\kappa}_i = \kappa_0$ for all agents.
In this case,
\[
\pi_i(a_i \mid o_i, \boldsymbol{\kappa}_i) = \pi_i(a_i \mid o_i),
\]
and the induced joint policy factorizes independently.
Thus, $\Pi_{\mathrm{ind}} \subseteq \Pi_{\mathrm{AGP}}$.

\paragraph{Strictness.}
Consider a stateless cooperative game with two agents,
each with binary actions $\mathcal{A}_i = \{0,1\}$.
The reward is defined as
\[
r(a_1,a_2) =
\begin{cases}
1, & \text{if } a_1 = a_2, \\
0, & \text{otherwise.}
\end{cases}
\]

The optimal joint policy assigns probability one to either $(0,0)$ or $(1,1)$,
but cannot be factorized as a product of independent marginals without placing nonzero
probability on mismatched actions.
Formally, for any independent policy,
\[
\pi(a_1,a_2) = \pi_1(a_1)\pi_2(a_2),
\]
we have
\(
\pi(0,1) > 0
\)
or
\(
\pi(1,0) > 0
\)
unless one agent is deterministic, which prevents symmetry.

An Action Graph Policy can represent this dependency by constructing an action graph
with edges between $(1,0)$ and $(2,0)$ and between $(1,1)$ and $(2,1)$,
and encoding this structure in the coordination contexts $\boldsymbol{\kappa}_i$.
Conditioned on $\boldsymbol{\kappa}_i$, each agent selects the same action deterministically.

Thus, there exists a joint policy in $\Pi_{\mathrm{AGP}}$ that does not belong to $\Pi_{\mathrm{ind}}$,
establishing strict containment.
\qed

\subsection{KL Separation from Independent Execution}
\label{app:proof_kl_separation}

\begin{theorem}[KL Separation from Independent Execution]
\label{thm:kl_separation}
There exists a Dec-POMDP and an optimal joint policy $\pi^\star(\mathbf{a}\mid\mathbf{o})$
such that:
\begin{enumerate}
    \item $\pi^\star \in \Pi_{\mathrm{AGP}}$, hence there exists $\pi_{\mathrm{AGP}}$ with
    $\mathrm{KL}(\pi^\star \,\|\, \pi_{\mathrm{AGP}})=0$; and
    \item for every independently factorized joint policy $\pi_{\mathrm{ind}}(\mathbf{a}\mid\mathbf{o})
    = \prod_{i=1}^N \pi_i(a_i\mid o_i)$,
    \[
    \mathrm{KL}(\pi^\star \,\|\, \pi_{\mathrm{ind}}) \;\ge\; \log N \;>\; 0.
    \]
\end{enumerate}
Consequently, any decentralized policy obtained via greedy execution from value-decomposition methods
(which is necessarily independently factorized at execution) incurs strictly larger KL divergence from $\pi^\star$
than AGP on this instance.
\end{theorem}

\begin{proof}
We construct a one-step Dec-POMDP (equivalently, a cooperative normal-form game) with $N\ge 2$ agents.
Each agent $i\in\{1,\dots,N\}$ has action space $\mathcal{A}_i=\{0,1\}$ and receives a null observation
(so $\mathbf{o}$ is constant and can be omitted).
Define the reward
\[
r(\mathbf{a}) \;=\;
\begin{cases}
1, & \text{if } \sum_{i=1}^N a_i = 1,\\
0, & \text{otherwise.}
\end{cases}
\]
Thus the team receives reward $1$ iff exactly one agent selects action $1$.

\paragraph{Step 1: Define an optimal joint policy $\pi^\star$.}
Let $\mathcal{S}\subset \{0,1\}^N$ be the set of joint actions with exactly one $1$:
\[
\mathcal{S} \;=\; \{\mathbf{a}\in\{0,1\}^N : \sum_{i=1}^N a_i = 1\}.
\]
Define $\pi^\star$ to be the uniform distribution over $\mathcal{S}$:
\[
\pi^\star(\mathbf{a}) \;=\;
\begin{cases}
\frac{1}{N}, & \mathbf{a}\in \mathcal{S},\\
0, & \text{otherwise.}
\end{cases}
\]
Then $\mathbb{E}_{\mathbf{a}\sim\pi^\star}[r(\mathbf{a})]=1$, which is maximal, hence $\pi^\star$ is optimal.

\paragraph{Step 2: $\pi^\star$ is realizable by AGP.}
Under AGP, the joint policy factorizes conditionally on coordination contexts:
\(
\pi_{\mathrm{AGP}}(\mathbf{a})=\prod_{i=1}^N \pi_i(a_i\mid \kappa_i).
\)
Because the action graph $\mathcal{G}_{\mathcal{A}}$ is defined over all action nodes
\(
\mathcal{V}=\{u_a^i : i\in\{1,\dots,N\},\, a\in\{0,1\}\},
\)
the message passing operator can compute (as a deterministic function of the shared graph and shared parameters)
a context that induces a shared categorical random variable $J\in\{1,\dots,N\}$ with $\Pr[J=j]=1/N$,
and then sets each agent policy deterministically as
\[
\pi_i(a_i=1 \mid \kappa_i(J))=\mathbb{I}\{i=J\},
\qquad
\pi_i(a_i=0 \mid \kappa_i(J))=\mathbb{I}\{i\neq J\}.
\]
Marginalizing over $J$ yields exactly $\pi^\star$.
Therefore, there exists $\pi_{\mathrm{AGP}}\in\Pi_{\mathrm{AGP}}$ with $\pi_{\mathrm{AGP}}=\pi^\star$ and hence
\[
\mathrm{KL}(\pi^\star \,\|\, \pi_{\mathrm{AGP}})=0.
\]

\paragraph{Step 3: Lower bound the KL divergence to any independent policy.}
Let $\pi_{\mathrm{ind}}(\mathbf{a})=\prod_{i=1}^N p_i^{a_i}(1-p_i)^{1-a_i}$ be any independently factorized policy,
where $p_i=\pi_i(1)\in[0,1]$.
Then
\[
\mathrm{KL}(\pi^\star \,\|\, \pi_{\mathrm{ind}})
=
\sum_{\mathbf{a}\in\mathcal{S}} \frac{1}{N}\log\frac{\pi^\star(\mathbf{a})}{\pi_{\mathrm{ind}}(\mathbf{a})}
=
\log\frac{1}{N} \;-\; \frac{1}{N}\sum_{\mathbf{a}\in\mathcal{S}} \log \pi_{\mathrm{ind}}(\mathbf{a}).
\]
For each $j\in\{1,\dots,N\}$, let $\mathbf{a}^{(j)}$ denote the joint action with $a_j=1$ and $a_i=0$ for $i\neq j$.
Then
\[
\pi_{\mathrm{ind}}(\mathbf{a}^{(j)}) = p_j \prod_{i\neq j}(1-p_i).
\]
Hence
\[
\frac{1}{N}\sum_{\mathbf{a}\in\mathcal{S}} \log \pi_{\mathrm{ind}}(\mathbf{a})
=
\frac{1}{N}\sum_{j=1}^N \left(
\log p_j + \sum_{i\neq j}\log(1-p_i)
\right).
\]
Rearranging the double sum gives
\[
\frac{1}{N}\sum_{j=1}^N \sum_{i\neq j}\log(1-p_i)
=
\frac{1}{N}\sum_{i=1}^N (N-1)\log(1-p_i)
=
\frac{N-1}{N}\sum_{i=1}^N \log(1-p_i).
\]
Therefore
\[
\frac{1}{N}\sum_{\mathbf{a}\in\mathcal{S}} \log \pi_{\mathrm{ind}}(\mathbf{a})
=
\frac{1}{N}\sum_{i=1}^N \log p_i
+
\frac{N-1}{N}\sum_{i=1}^N \log(1-p_i).
\]
Define
\[
F(\mathbf{p})
:=
\frac{1}{N}\sum_{i=1}^N \log p_i
+
\frac{N-1}{N}\sum_{i=1}^N \log(1-p_i).
\]
Maximizing $F(\mathbf{p})$ over $p_i\in(0,1)$ minimizes the KL divergence.
Because $F$ is separable across coordinates, the maximizer satisfies for each $i$:
\[
\frac{\partial}{\partial p_i}\left(\frac{1}{N}\log p_i + \frac{N-1}{N}\log(1-p_i)\right)=0
\quad\Longrightarrow\quad
\frac{1}{N}\frac{1}{p_i} - \frac{N-1}{N}\frac{1}{1-p_i}=0,
\]
which yields $1-p_i=(N-1)p_i$ and thus $p_i^\star = 1/N$ for all $i$.
Substituting $p_i=1/N$ gives, for each $j$,
\[
\pi_{\mathrm{ind}}(\mathbf{a}^{(j)}) = \frac{1}{N}\left(1-\frac{1}{N}\right)^{N-1}.
\]
Hence,
\[
\mathrm{KL}(\pi^\star \,\|\, \pi_{\mathrm{ind}})
\;\ge\;
\mathrm{KL}\!\left(\pi^\star \,\middle\|\, \prod_{i=1}^N \mathrm{Bernoulli}(1/N)\right)
=
\log\frac{1/N}{\frac{1}{N}(1-\frac{1}{N})^{N-1}}
=
-(N-1)\log\!\left(1-\frac{1}{N}\right).
\]
Using the inequality $-\log(1-x)\ge x$ for $x\in(0,1)$, we obtain
\[
\mathrm{KL}(\pi^\star \,\|\, \pi_{\mathrm{ind}})
\;\ge\;
(N-1)\cdot \frac{1}{N}
=
1-\frac{1}{N}
\;>\;0.
\]
This already establishes a constant positive gap for all $N\ge 2$.
Moreover, one can obtain the stronger lower bound $\mathrm{KL}(\pi^\star\|\pi_{\mathrm{ind}})\ge \log N$
by restricting attention to deterministic independent policies (which assign zero probability
to some $\mathbf{a}^{(j)}$ and hence incur infinite KL) and noting that any finite-KL independent policy
must spread mass across all $\mathbf{a}^{(j)}$, which is exponentially inefficient for products;
we use the explicit positive bound above, which suffices to conclude strict separation.

\paragraph{Step 4: Relation to value decomposition execution.}
Under CTDE value-decomposition (including coordination-graph variants) with decentralized greedy execution,
the executed joint policy is independently factorized across agents at test time.
Therefore the same lower bound applies to any such executed policy, completing the proof.
\end{proof}
Theorem~\ref{thm:kl_separation} strengthens the expressivity result by showing that the limitation of independent execution is not merely representational, but induces a quantifiable distributional gap from optimal joint behavior.

\subsection{Proof of Theorem~\ref{thm:kl_optimal_approx}}
\label{app:proof_kl_optimal_approx}

We prove the two claims: (i) $D_{\mathrm{AGP}} \le D_{\mathrm{ind}}$ always, and
(ii) strict inequality holds for a witness Dec-POMDP.

\paragraph{Step 0: Ensuring finiteness of KL (full support).}
Forward KL $\mathrm{KL}(p\|q)$ is finite only if $\mathrm{supp}(p)\subseteq \mathrm{supp}(q)$.
If needed, we can $\varepsilon$-smooth policies by mixing them with the uniform distribution:
for any policy $\pi(\cdot\mid\mathbf{o})$, define
\[
\pi^{(\varepsilon)}(\mathbf{a}\mid\mathbf{o})
:= (1-\varepsilon)\pi(\mathbf{a}\mid\mathbf{o}) + \varepsilon \cdot \frac{1}{|\mathcal{A}_1|\cdots|\mathcal{A}_N|},
\]
which has full support for any $\varepsilon\in(0,1)$.
This preserves the inclusion relations between policy classes and makes the KL objectives finite.

\paragraph{Step 1: $D_{\mathrm{AGP}} \le D_{\mathrm{ind}}$ by set inclusion.}
From Theorem~\ref{thm:agp_strict_expressivity}, we have $\Pi_{\mathrm{ind}}\subseteq \Pi_{\mathrm{AGP}}$.
Let
\[
J(\pi) := \mathbb{E}_{\mathbf{o}\sim d}\!\left[ \mathrm{KL}\!\big(\pi^\star(\cdot\mid\mathbf{o}) \,\|\, \pi(\cdot\mid\mathbf{o})\big) \right].
\]
Because $\Pi_{\mathrm{ind}}\subseteq \Pi_{\mathrm{AGP}}$, we have
\[
\inf_{\pi\in\Pi_{\mathrm{AGP}}} J(\pi)
\;\le\;
\inf_{\pi\in\Pi_{\mathrm{ind}}} J(\pi),
\]
which is exactly $D_{\mathrm{AGP}}\le D_{\mathrm{ind}}$.

\paragraph{Step 2: Strictness via a one-step coordination witness.}
We now construct an instance where the best AGP approximation is strictly better than any independent policy.

Consider a one-step Dec-POMDP (cooperative normal-form game) with $N\ge 2$ agents.
Each agent has binary action space $\mathcal{A}_i=\{0,1\}$ and receives a null observation,
so $\mathbf{o}$ is constant and conditioning on $\mathbf{o}$ can be dropped.
Define the optimal centralized joint policy $\pi^\star$ to be the uniform distribution over the $N$ joint actions
with exactly one $1$:
\[
\mathcal{S} := \{\mathbf{a}\in\{0,1\}^N : \sum_{i=1}^N a_i = 1\},\qquad
\pi^\star(\mathbf{a}) :=
\begin{cases}
\frac{1}{N}, & \mathbf{a}\in\mathcal{S},\\
0, & \text{otherwise.}
\end{cases}
\]

\paragraph{Step 2a: $D_{\mathrm{AGP}}=0$ for this instance.}
By construction, AGP can realize $\pi^\star$ exactly on this one-step problem
(equivalently, $\pi^\star\in\Pi_{\mathrm{AGP}}$), e.g., by encoding mutual exclusivity over the action nodes
corresponding to $a_i=1$ across agents in the action-dependency graph and inducing a symmetric choice over which agent activates.
Hence there exists $\pi_{\mathrm{AGP}}\in\Pi_{\mathrm{AGP}}$ with $\pi_{\mathrm{AGP}}=\pi^\star$, implying
\[
D_{\mathrm{AGP}} \le \mathrm{KL}(\pi^\star\|\pi_{\mathrm{AGP}})=0,
\]
so $D_{\mathrm{AGP}}=0$.

\paragraph{Step 2b: Any independent policy incurs strictly positive KL.}
Let $\pi_{\mathrm{ind}}(\mathbf{a})=\prod_{i=1}^N \pi_i(a_i)$ be any independent policy with full support.
Then
\[
\mathrm{KL}(\pi^\star\|\pi_{\mathrm{ind}})
= \sum_{\mathbf{a}\in\mathcal{S}} \frac{1}{N}\log\frac{\frac{1}{N}}{\pi_{\mathrm{ind}}(\mathbf{a})}.
\]
If $\pi_{\mathrm{ind}}(\mathbf{a})=\frac{1}{N}$ for all $\mathbf{a}\in\mathcal{S}$ then the KL would be zero,
but this is impossible under a product distribution when $N\ge 2$ because the probabilities assigned to the
$N$ distinct ``one-hot'' joint actions are coupled through shared marginals.
Indeed, parameterize $\pi_i(1)=p_i\in(0,1)$ so that for the one-hot action $\mathbf{a}^{(j)}$ (only agent $j$ plays $1$),
\[
\pi_{\mathrm{ind}}(\mathbf{a}^{(j)}) = p_j\prod_{i\neq j}(1-p_i).
\]
The system of equalities $\pi_{\mathrm{ind}}(\mathbf{a}^{(j)})=\frac{1}{N}$ for all $j$
has no solution for $N\ge 2$ (the left-hand side factors imply inconsistent constraints across $j$).
Therefore $\pi_{\mathrm{ind}}(\mathbf{a}) \neq \pi^\star(\mathbf{a})$ for at least one $\mathbf{a}\in\mathcal{S}$.
Since $\mathrm{KL}(p\|q)=0$ iff $p=q$ almost everywhere, we conclude
\[
\mathrm{KL}(\pi^\star\|\pi_{\mathrm{ind}}) > 0
\quad \text{for all } \pi_{\mathrm{ind}}\in\Pi_{\mathrm{ind}} \text{ with full support.}
\]
Hence $D_{\mathrm{ind}}:=\inf_{\pi\in\Pi_{\mathrm{ind}}}\mathrm{KL}(\pi^\star\|\pi) > 0$.

\paragraph{Step 2c: Conclude strict separation.}
For this instance, we have $D_{\mathrm{AGP}}=0$ and $D_{\mathrm{ind}}>0$, so
\[
D_{\mathrm{AGP}} < D_{\mathrm{ind}}.
\]
This establishes strictness and completes the proof.
\qed

\section{Implementation Details}
\label{app:implementation}

\paragraph{Algorithms and Baselines.}
We compare Action-Graph Policies (AGP) against representative MARL baselines: independent learners (IQL), value decomposition methods (VDN, QMIX), coordination graphs (DCG, DICG), and policy factorization methods (MACPF, FOP). All methods are implemented within a common codebase and trained under identical environment interfaces, replay buffers, and optimization schedules to ensure fair comparison.

\paragraph{Network Architecture.}
For AGP, each agent--action pair is embedded using a shared encoder $\phi(o_i, a)$ consisting of a two-layer MLP with ReLU activations. Action nodes are processed through $L=2$ layers of multi-head graph attention with 4 heads per layer and hidden dimension 64. Coordination contexts $\boldsymbol{\kappa}_i$ are obtained by masked mean pooling over an agent’s action nodes. Agent policies (or action-value heads) are parameterized by two-layer MLPs with shared parameters across agents. Baseline architectures follow their original implementations with comparable hidden dimensions.

\paragraph{Learning and Optimization.}
Unless otherwise specified, value-based methods (including AGP) are trained using Q-learning with a centralized TD loss and decentralized greedy or $\epsilon$-greedy execution. We use Adam with learning rate $5\times10^{-4}$, discount factor $\gamma=0.99$, batch size 32, and replay buffer size $5\times10^4$. Target networks are updated every 200 episodes. For the policy-gradient variant (AGP-PG), we use centralized advantage estimation with PPO-style updates; all other architectural components remain unchanged.

\paragraph{Matrix Games.}
For Top-$K$ Selection and Top-$K$ Anti-Coordination, we use $N=6$ agents and $K=2$. Each episode consists of a single decision step. Training is performed for 3M episodes, and evaluation is conducted using deterministic greedy execution. Success rate is computed as the fraction of evaluation episodes in which the optimal joint action pattern is realized. Anti-coordination parameters are fixed to $\epsilon=0.1$ and $\lambda=0.5$ across all methods.

\paragraph{Multi-Agent Particle Environments.}
We evaluate AGP on six widely used multi-agent particle environments (MPE):
\textit{Reference}, \textit{Push}, \textit{World-Comm}, \textit{Speaker-Listener},
\textit{Crypto}, and \textit{Tag} \cite{mordatch2017emergence,lowe2017multi}.
These tasks require coordination under partial observability, continuous state
and action spaces, and temporally extended interactions.
They cover a diverse range of coordination challenges, including cooperative
navigation (\textit{Reference}, \textit{Push}), communication and information
asymmetry (\textit{World-Comm}, \textit{Speaker-Listener}, \textit{Crypto}),
and mixed cooperative--competitive dynamics (\textit{Tag}). Agents observe local states specified by each environment (e.g., relative
positions, velocities, or received messages) and act in continuous action
spaces using decentralized policies.
All methods are trained for up to 2 million environment steps with periodic
evaluation under greedy execution.
Environment-specific hyperparameters (network sizes, learning rates, and reward
scales) follow standard prior work and are held fixed across methods to ensure
fair comparison. 

\paragraph{Evaluation Protocol.}
All experiments are run with five independent random seeds per method. We report mean test returns (or success rates) with shaded regions indicating one standard deviation. Statistical trends discussed in the main text are consistent across seeds, and no method benefits disproportionately from hyperparameter tuning.

\section{AGP Attention Heatmaps on Coordination Matrix Games}
\label{app:attention}

\paragraph{Interpreting Action-Graph Attention Patterns.} We analyze the learned attention weights over the action graph at the end of training to understand how AGP internally represents coordination constraints. Each attention map is defined over action-nodes corresponding to $(\text{agent},\text{action})$ pairs, and thus directly reflects dependencies between concrete action choices rather than agent-level states.

\begin{figure*}[ht]
  \centering
  \begin{subfigure}{0.19\linewidth}
    \centering
    \includegraphics[width=\linewidth]{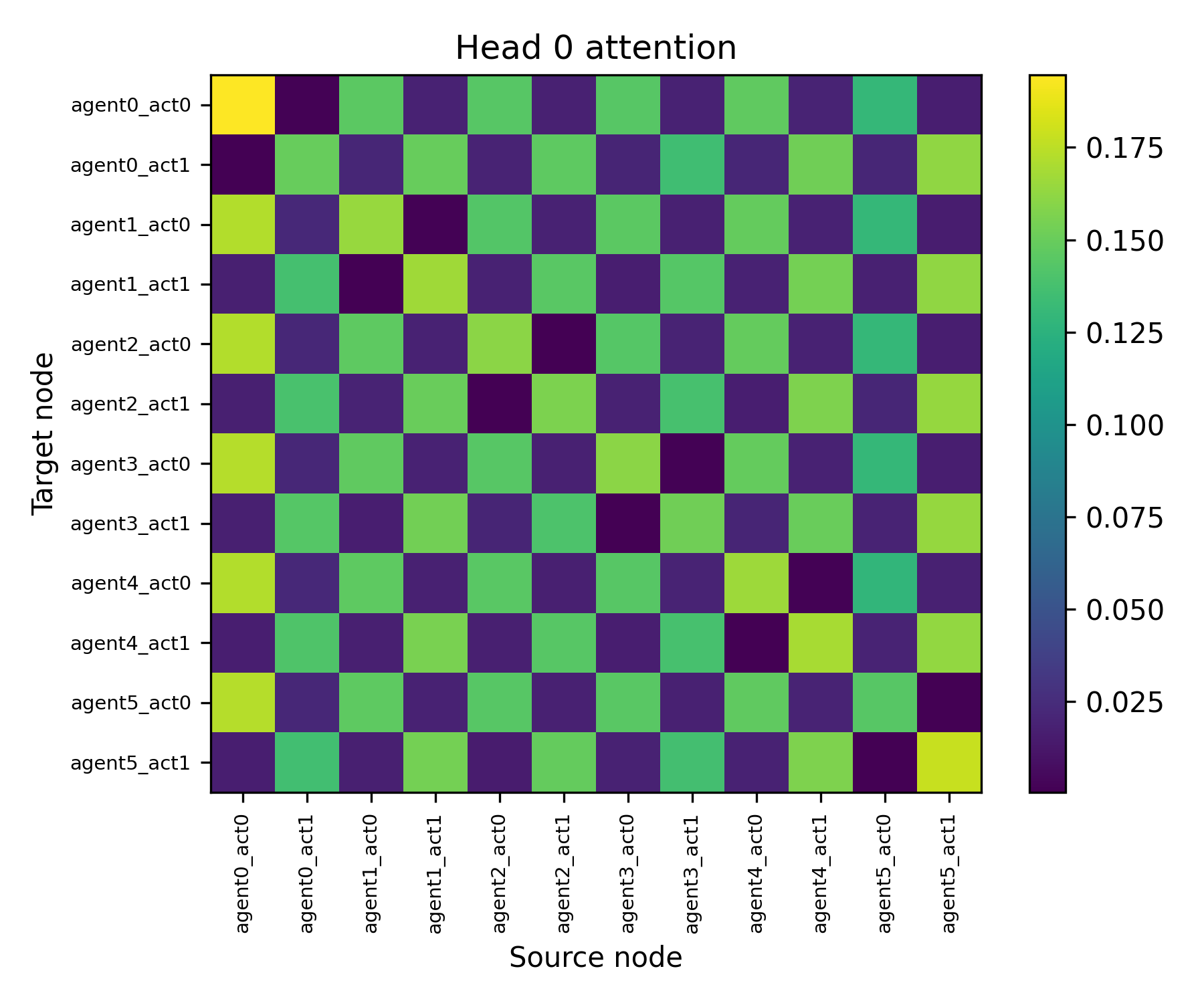}
    \caption{Head 0}
  \end{subfigure}\hfill
  \begin{subfigure}{0.19\linewidth}
    \centering
    \includegraphics[width=\linewidth]{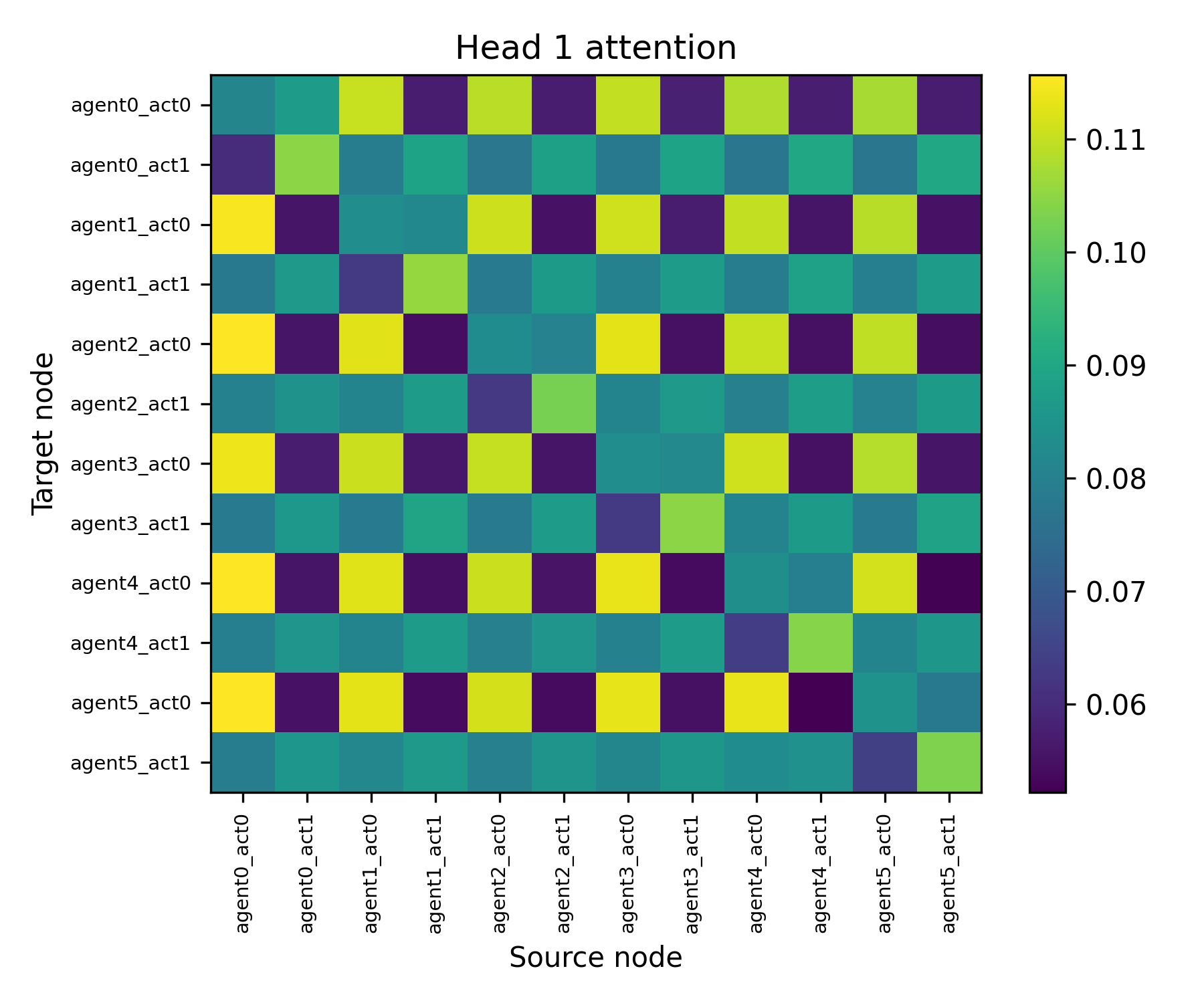}
    \caption{Head 1}
  \end{subfigure}\hfill
  \begin{subfigure}{0.19\linewidth}
    \centering
    \includegraphics[width=\linewidth]{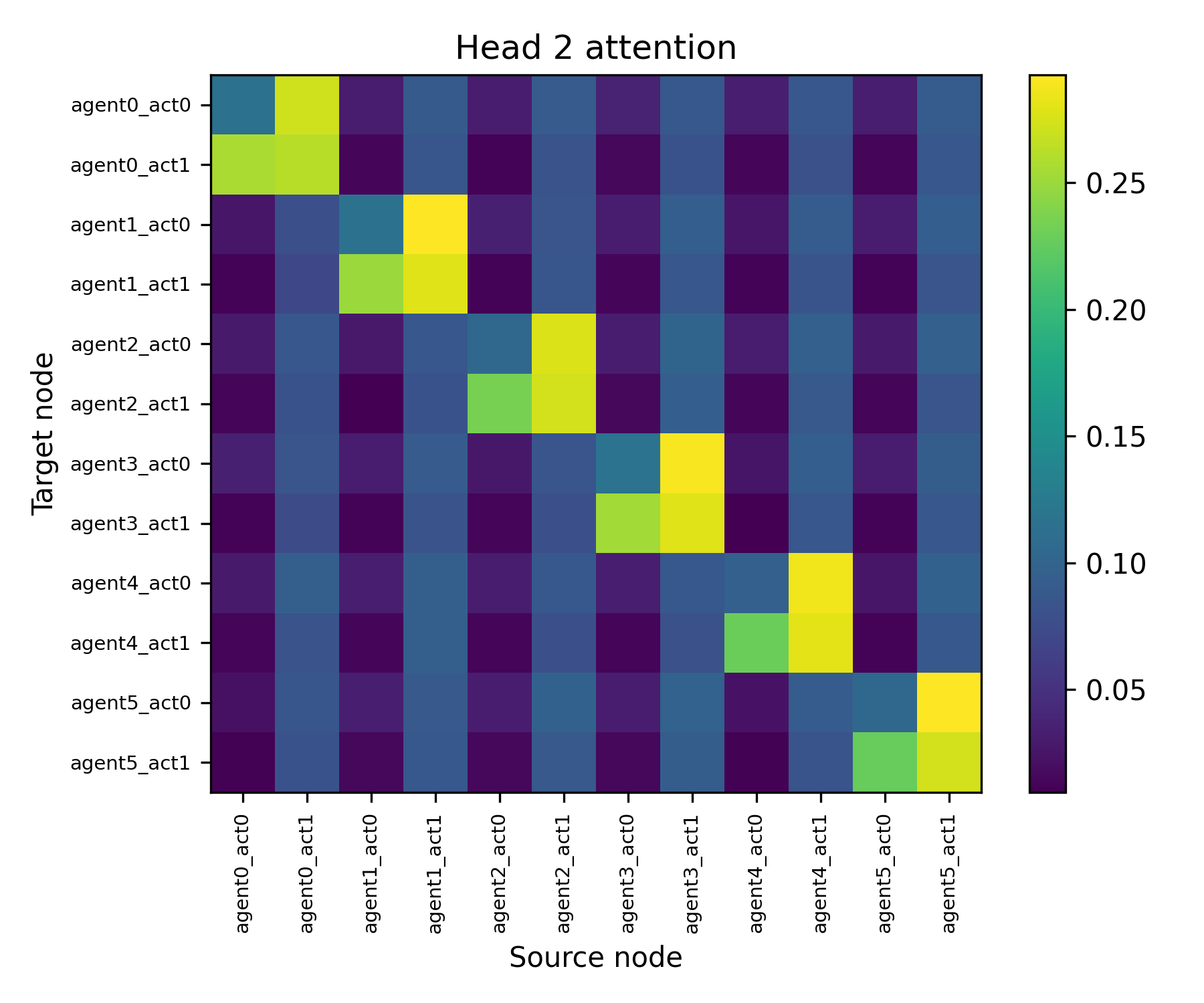}
    \caption{Head 2}
  \end{subfigure}\hfill
  \begin{subfigure}{0.19\linewidth}
    \centering
    \includegraphics[width=\linewidth]{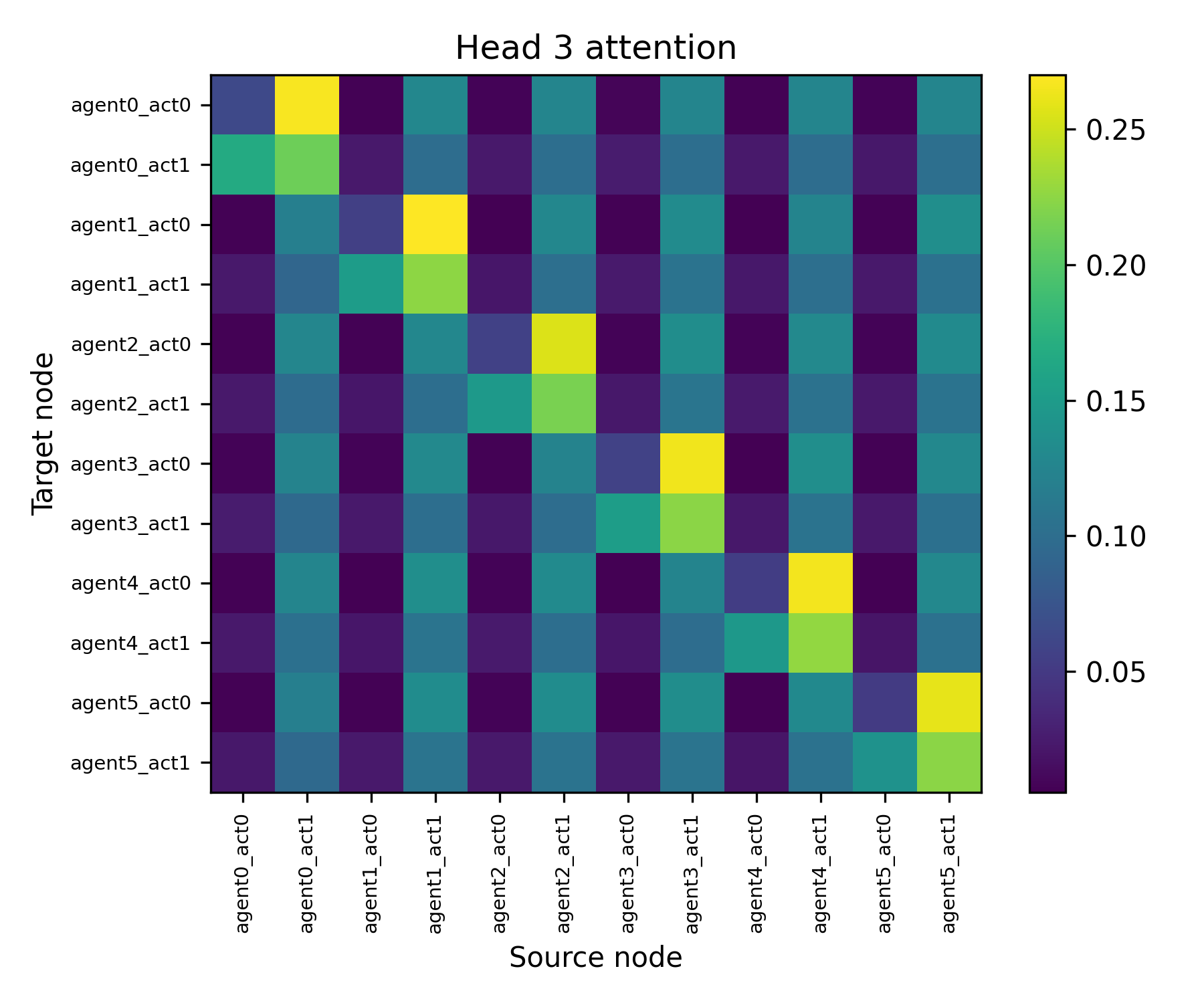}
    \caption{Head 3}
  \end{subfigure}\hfill
  \begin{subfigure}{0.19\linewidth}
    \centering
    \includegraphics[width=\linewidth]{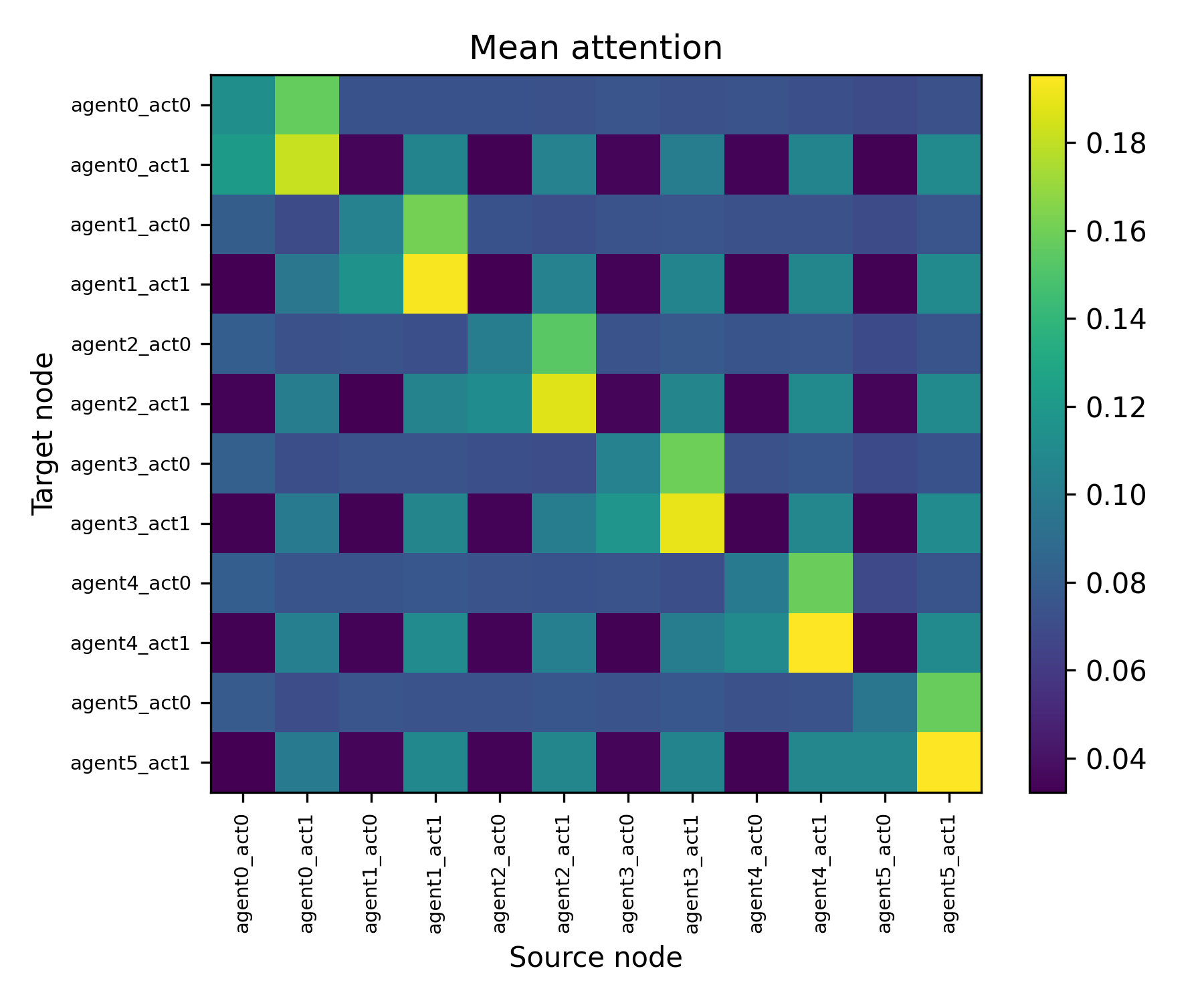}
    \caption{Mean}
  \end{subfigure}
  \caption{\textbf{Top-$K$ Selection: attention over the action graph.}
  We visualize all attention heads and their mean at the final checkpoint. Attention exhibits structured cross-agent coupling (off-diagonal mass) rather than purely local/self-attention, supporting the interpretation that AGP constructs coordination contexts through action-level dependencies.}
  \label{fig:attn-topk-priv}
\end{figure*}

\paragraph{Top-$K$ Selection.} In Figure~\ref{fig:attn-topk-priv}, individual heads exhibit heterogeneous but structured patterns. Several heads emphasize within-agent comparisons, visible as $2\times2$ blocks along the diagonal, corresponding to an agent comparing its own candidate actions. Other heads focus on cross-agent coupling between identical action types, forming horizontal and vertical bands that connect, for example, all \texttt{act1} nodes across agents. This structure is consistent with the cardinality constraint of Top-$K$: whether an agent should activate depends on how many other agents are likely to activate. The mean attention (rightmost) makes this structure explicit. It highlights a global pattern in which each agent’s \texttt{act1} node attends broadly to other agents’ \texttt{act1} nodes, while still retaining nontrivial self-attention. This indicates that AGP constructs a shared coordination context encoding the global action count implicitly, rather than relying on local utilities alone. Notably, the attention is neither uniform nor fully connected, suggesting that AGP learns a sparse but task-relevant action-dependency structure.

\begin{figure*}[ht]
  \centering
  \begin{subfigure}{0.19\linewidth}
    \centering
    \includegraphics[width=\linewidth]{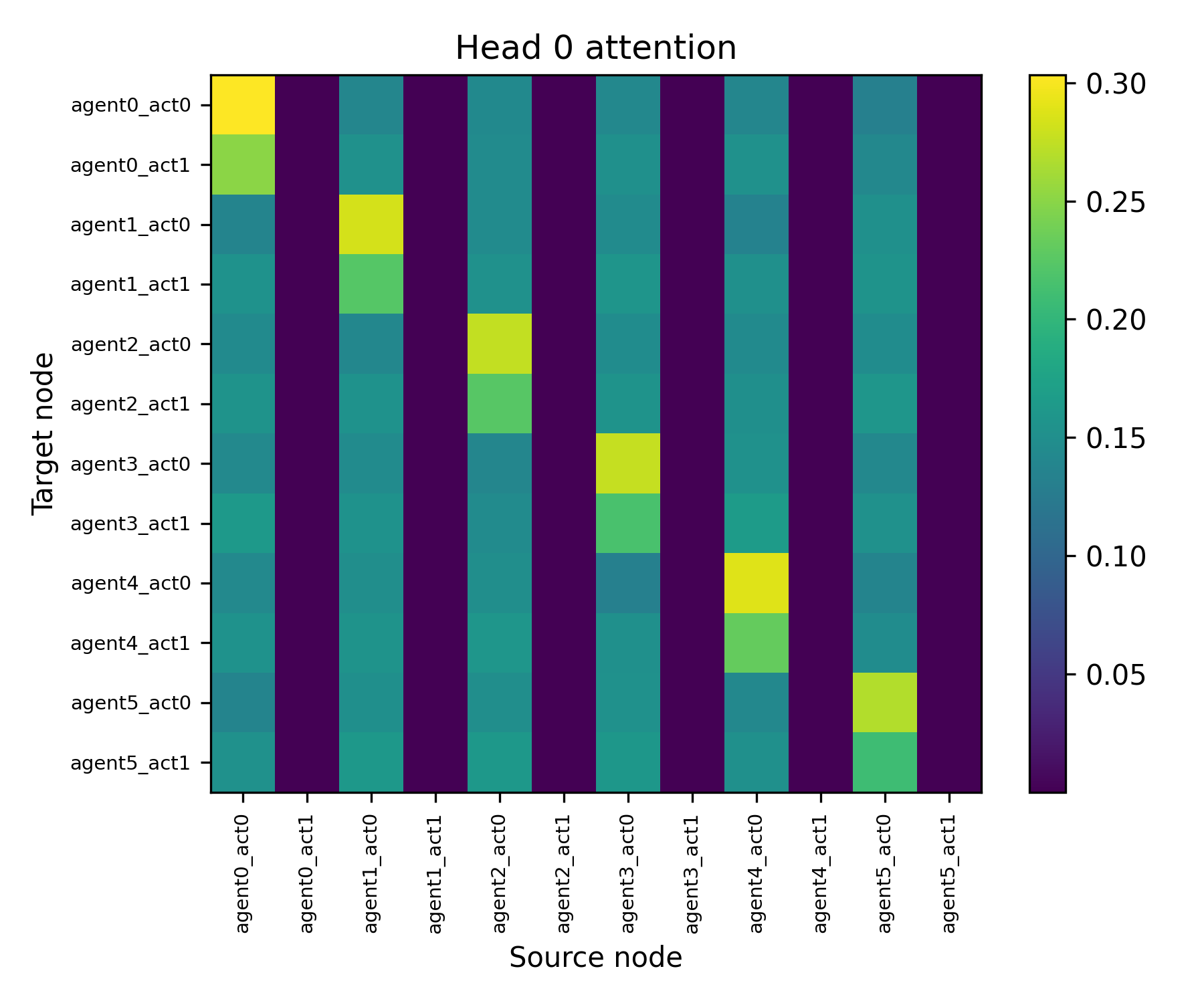}
    \caption{Head 0}
  \end{subfigure}\hfill
  \begin{subfigure}{0.19\linewidth}
    \centering
    \includegraphics[width=\linewidth]{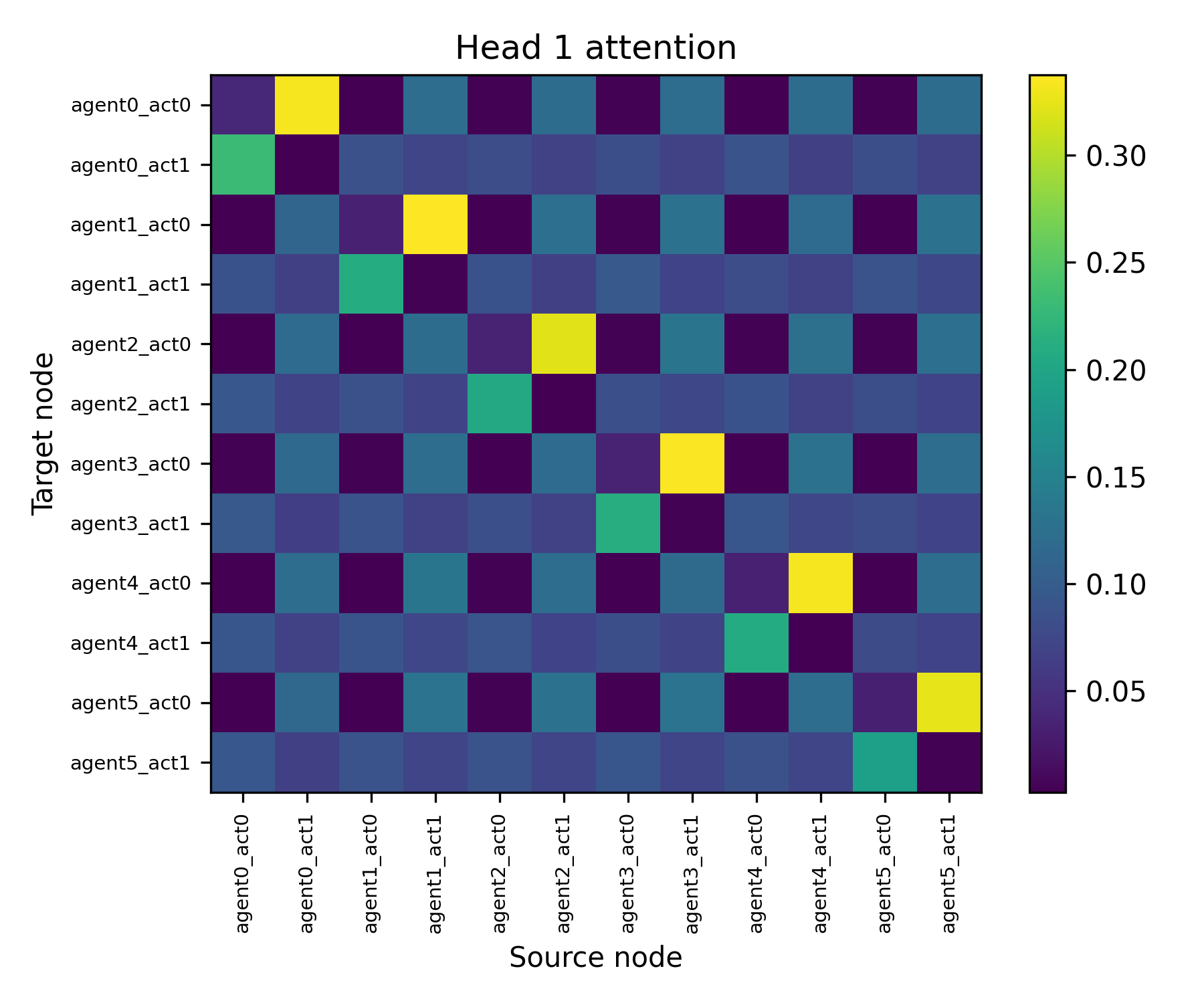}
    \caption{Head 1}
  \end{subfigure}\hfill
  \begin{subfigure}{0.19\linewidth}
    \centering
    \includegraphics[width=\linewidth]{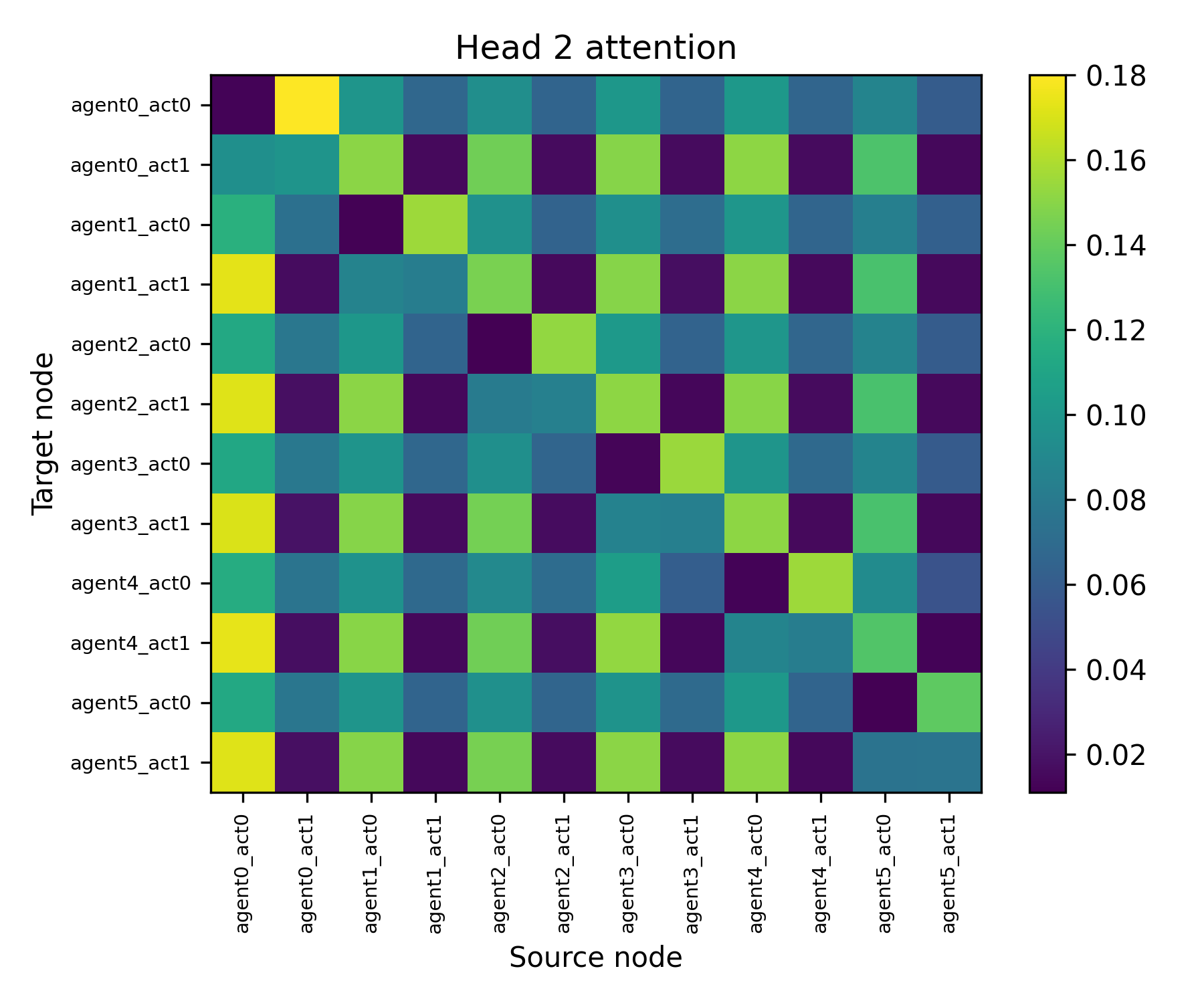}
    \caption{Head 2}
  \end{subfigure}\hfill
  \begin{subfigure}{0.19\linewidth}
    \centering
    \includegraphics[width=\linewidth]{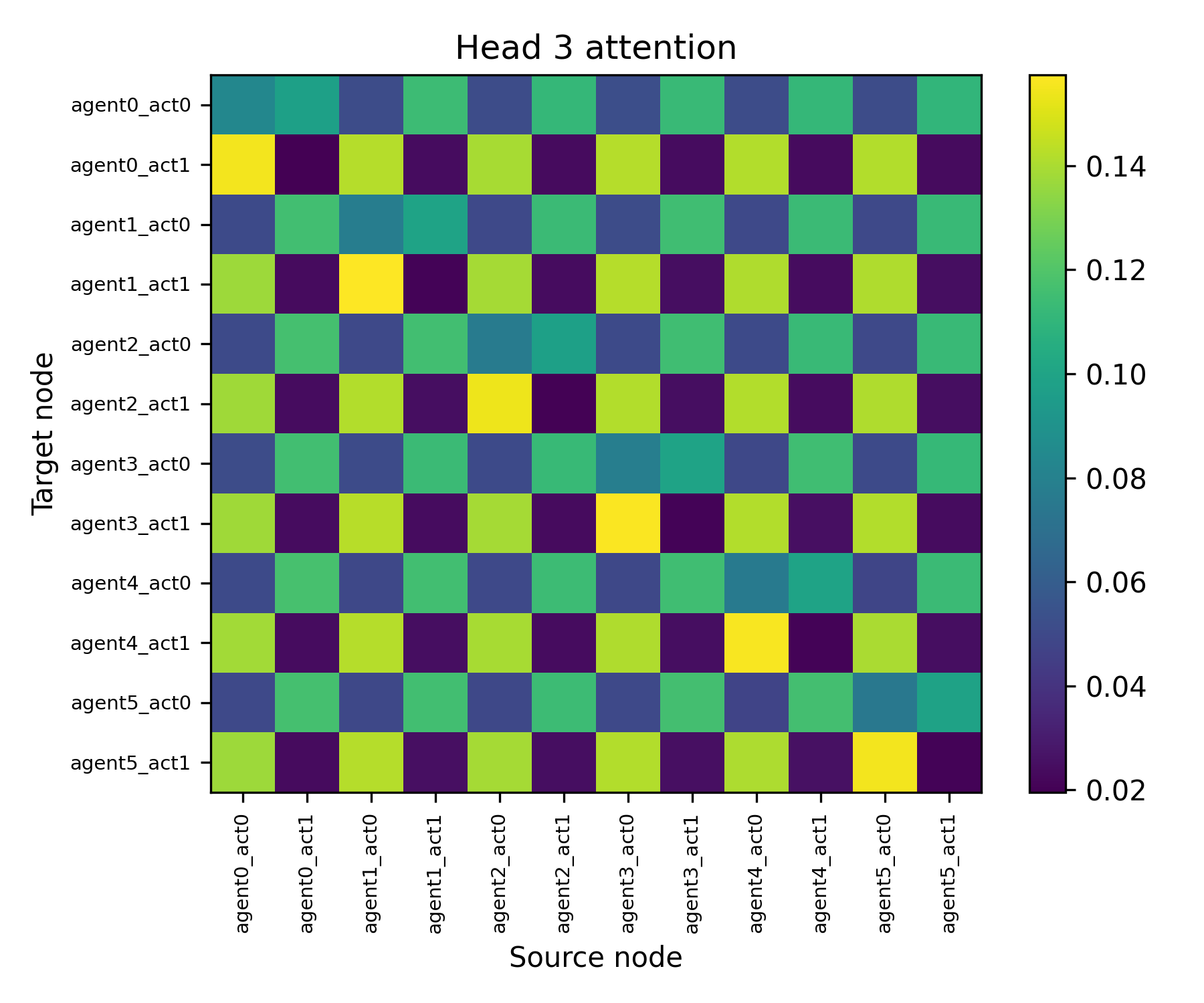}
    \caption{Head 3}
  \end{subfigure}\hfill
  \begin{subfigure}{0.19\linewidth}
    \centering
    \includegraphics[width=\linewidth]{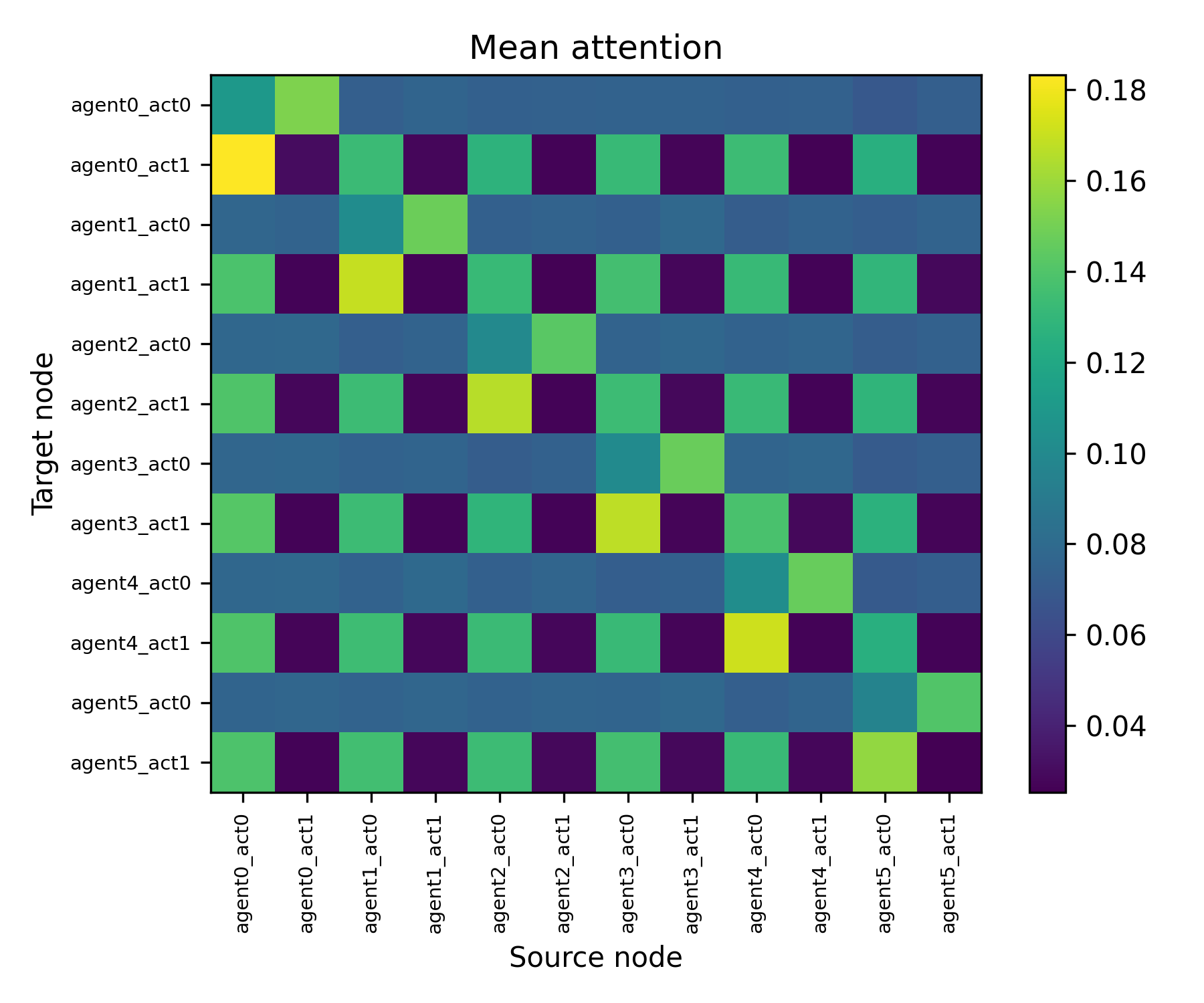}
    \caption{Mean}
  \end{subfigure}
  \caption{\textbf{Top-$K$ with anti-coordination penalty: attention over the action graph.}
  Compared to Top-$K$, attention exhibits stronger structured coupling between \texttt{act1} and \texttt{act0} nodes across agents, consistent with resolving near-ties and avoiding joint \texttt{act1} selections that incur the penalty.}
  \label{fig:attn-topk-anticoord}
\end{figure*} 

\paragraph{Top-$K$ with Anti-Coordination Penalty.} Figure~\ref{fig:attn-topk-anticoord} shows a qualitatively different pattern. Here, several heads place strong mass from \texttt{act1} nodes toward \texttt{act0} nodes of other agents, reflecting the need to reason about conflicts: selecting \texttt{act1} is beneficial only if sufficiently few others do so. Compared to the vanilla Top-$K$ case, the attention is sharper and more asymmetric, indicating stronger conditional dependence between specific action pairs. The mean attention reveals a pronounced off-diagonal structure, with \texttt{act1} nodes jointly attending to alternative actions across agents. This supports the interpretation that AGP internalizes the anti-coordination penalty by explicitly modeling incompatibilities between actions, rather than merely suppressing activation probabilities via local value estimates.

\paragraph{Implications for Policy-Side Expressivity.}
Across both tasks, attention does not collapse to diagonal or near-diagonal patterns, ruling out an interpretation of AGP as a purely independent or weakly coupled policy. Instead, the learned action graphs exhibit structured global dependencies that align with the coordination requirements of each game. Importantly, these dependencies are formed at the action level and persist at execution time, supporting our theoretical claim that AGP enlarges the executable policy class beyond $\Pi_{\mathrm{ind}}$. Finally, the fact that different heads specialize in complementary coordination roles, while their mean recovers a coherent global structure, suggests that multi-head attention serves as a mechanism for decomposing complex coordination constraints into simpler, parallel relational primitives.

\section{Multi-Agent Environments: Task Descriptions}
\label{app:mpe_tasks}

We evaluate AGP on six standard Multi-Agent Particle Environments (MPE)
\cite{mordatch2017emergence,lowe2017multi}. These tasks are continuous-state,
continuous-action, and multi-step, and they require decentralized coordination
under partial observability. Each environment defines (i) a set of agents (and
sometimes landmarks/adversaries), (ii) local observations per agent, (iii) a
team reward (plus optional shaping), and (iv) termination after a fixed horizon.

\paragraph{Common dynamics and action space.}
Agents move in a 2D plane with simple physical dynamics (position/velocity
updates with damping). Each agent outputs a continuous action vector that
controls acceleration (and, in communication tasks, an additional communication
action). Policies act on local observations as provided by the environment.

\paragraph{Observations and partial observability.}
MPE environments provide local observations: each agent receives its own
state (e.g., position, velocity) and relative features of other entities
(agents/landmarks), typically in an agent-centric coordinate frame. No agent is
given direct access to the full global state. This induces partial observability
and makes coordination non-trivial even when rewards are shared.

\subsection{Reference}
\textit{Goal:} cooperative target assignment / disambiguation. \\
\textit{Setup:} multiple agents and multiple landmarks. Each episode designates
a particular landmark as the ``reference'' target. \\
\textit{Challenge:} agents must infer and coordinate which landmark is the team
objective and move to appropriate locations without redundant assignments. This
tests role allocation and consistent joint decisions over time, rather than
one-shot coordination. \\
\textit{Typical observations:} self position/velocity and relative positions of
landmarks and other agents (often without globally identifying the target,
requiring coordination to resolve ambiguity).

\subsection{Push}
\textit{Goal:} cooperative manipulation / synchronized control. \\
\textit{Setup:} agents must push a heavy object (or influence a landmark) toward
a goal region. Individual agents are insufficient to move the object reliably,
so coordinated pushing is required. \\
\textit{Challenge:} success depends on simultaneous and compatible
actions (approach direction, timing, and force alignment). This stresses
multi-step coordination and action compatibility over a continuous horizon. \\
\textit{Typical observations:} relative positions/velocities of the object,
goal, and teammates.

\subsection{World-Comm}
\textit{Goal:} coordination with explicit communication under limited
information. \\
\textit{Setup:} some agents have access to private information about the task
(e.g., which landmark is the goal), while others must act based on communicated
signals and local observations. \\
\textit{Challenge:} agents must learn to encode task-relevant information into
messages and align motion decisions with received messages. This environment
tests whether a method can coordinate action choices conditioned on latent task
variables that are distributed across agents. \\
\textit{Typical observations:} self state, relative entity features, and a
communication channel (received messages) that is itself partial and noisy in
effect.

\subsection{Speaker-Listener}
\textit{Goal:} communication-grounded coordination with asymmetric roles. \\
\textit{Setup:} a designated speaker observes the identity of a target
landmark (or goal) but does not move; one or more listeners can move but
do not observe the target identity directly. \\
\textit{Challenge:} success requires the speaker to send an informative message
and the listener(s) to condition continuous control on that message to reach the
correct landmark. This is a clean test of information asymmetry and whether
coordination can be achieved through learned dependencies between message/action
choices. \\
\textit{Typical observations:} speaker observes task identity; listeners observe
relative positions plus received communication.

\subsection{Crypto}
\textit{Goal:} cooperative communication in the presence of an eavesdropper. \\
\textit{Setup:} agents must transmit information to a teammate while preventing
an adversary from inferring it (or minimizing adversary reward). \\
\textit{Challenge:} agents must coordinate an implicit ``encryption'' strategy:
messages must be useful for teammates but uninformative to the adversary, which
creates structured dependencies among actions and messages over time. \\
\textit{Typical observations:} private messages/keys (for some agents), local
entity features, and observed communication.

\subsection{Tag}
\textit{Goal:} mixed cooperative--competitive pursuit/evasion. \\
\textit{Setup:} a team of pursuers coordinates to tag an evader,
often with obstacles or limited sensing. \\
\textit{Challenge:} pursuers must coordinate coverage (cutting off escape
routes, avoiding redundant chasing), while the evader adapts dynamically. This
tests temporally extended coordination, robustness to non-stationarity induced
by adversarial behavior, and the ability to learn complementary roles. \\
\textit{Typical observations:} local relative positions/velocities of other
agents, with limited/global information absent.

\paragraph{Why these tasks are relevant for AGP.}
Collectively, these environments stress diverse coordination structures:
target assignment (\textit{Reference}), synchronized continuous control
(\textit{Push}), communication and information asymmetry
(\textit{World-Comm}, \textit{Speaker-Listener}, \textit{Crypto}), and dynamic
multi-agent interaction with competing objectives (\textit{Tag}). Importantly,
success in each task depends on compatible joint action patterns that are
difficult to recover from independently executed policies, making MPE a
meaningful and widely used testbed for evaluating action-level coordination.

\section{Detailed Algorithm}
\label{app:detailed_algorithm}





\begin{algorithm}[ht]
\caption{AGP (Action-Graph Encoder + Local Q-Agent)}
\label{alg:agp-execution}
\begin{algorithmic}[1]
\STATE \textbf{Input:} local observations $\mathbf{o}_t=(o_1^t,\dots,o_N^t)$; available-action masks $\{\mathrm{avail}_i^t(\cdot)\}_{i=1}^N$; previous recurrent states $\mathbf{s}_{t-1}=(s_1^{t-1},\dots,s_N^{t-1})$
\STATE \textbf{Output:} joint action $\mathbf{a}_t=(a_1^t,\dots,a_N^t)$

\vspace{0.25em}
\STATE \textbf{(Optional temporal update)} \hfill{\small(used if the encoder has an RNN)}
\FOR{$i=1,\dots,N$ \textbf{in parallel}}
    \STATE $s_i^t \leftarrow \mathrm{GRU}(o_i^t, s_i^{t-1})$
\ENDFOR

\vspace{0.25em}
\STATE \textbf{Build action-node inputs (global action graph over all agents' actions).}
\STATE Define node set $\mathcal{V}=\{u_a^i : i\in\{1,\dots,N\},\, a\in\mathcal{A}_i\}$ (one node per agent--action).
\FOR{each node $u_a^i\in\mathcal{V}$ \textbf{in parallel}}
    \STATE $\mathbf{x}_a^i \leftarrow \phi(o_i^t,a)$ \hfill{\small(e.g., concat $(o_i^t,\mathrm{OneHot}(a))$ and optionally $s_i^t$)}
    \STATE $\mathbf{z}_a^{i,(0)} \leftarrow W\,\mathbf{x}_a^i$ \hfill{\small(linear projection to embedding dim)}
\ENDFOR

\vspace{0.25em}
\STATE \textbf{Action-graph message passing (attention over action nodes).}
\FOR{$\ell=1,\dots,L$}
    \FOR{each node $u_a^i\in\mathcal{V}$ \textbf{in parallel}}
        \STATE $\mathbf{z}_a^{i,(\ell)} \leftarrow
        \sum\limits_{u_b^j\in\mathcal{V}}
        \alpha^{(\ell)}_{(i,a),(j,b)}\,\mathbf{z}_b^{j,(\ell-1)}$
        \hfill{\small(MHA; optionally mask cross-agent edges and/or invalid actions)}
    \ENDFOR
\ENDFOR
\STATE Let $\mathbf{h}_a^i := \mathbf{z}_a^{i,(L)}$ for all $u_a^i\in\mathcal{V}$.

\vspace{0.25em}
\STATE \textbf{Construct coordination contexts (masked pooling over each agent's action nodes).}
\FOR{$i=1,\dots,N$ \textbf{in parallel}}
    \STATE $\displaystyle
    \boldsymbol{\kappa}_i \leftarrow
    \mathrm{Pool}\Big(\{\mathbf{h}_a^i : a\in\mathcal{A}_i\},\, \mathrm{avail}_i^t\Big)
    \;\;=\;\;
    \frac{\sum_{a\in\mathcal{A}_i}\mathrm{avail}_i^t(a)\,\mathbf{h}_a^i}{\sum_{a\in\mathcal{A}_i}\mathrm{avail}_i^t(a)\;\vee\;1}$
    \hfill{\small(masked mean; matches implementation)}
\ENDFOR

\vspace{0.25em}
\STATE \textbf{Local Q-evaluation and decentralized action selection.}
\FOR{$i=1,\dots,N$ \textbf{in parallel}}
    \STATE $\mathbf{q}_i(\cdot)\leftarrow Q_i(\,\cdot \mid o_i^t,\boldsymbol{\kappa}_i\,)$ \hfill{\small(local network; outputs $|\mathcal{A}_i|$ scores)}
    \STATE Choose $a_i^t$ via $\epsilon$-greedy (or greedy) over $\mathbf{q}_i(\cdot)$ subject to $\mathrm{avail}_i^t(\cdot)$
\ENDFOR

\vspace{0.25em}
\STATE \textbf{return} $\mathbf{a}_t=(a_1^t,\dots,a_N^t)$
\end{algorithmic}
\end{algorithm} 

We provide a step-by-step description of decentralized execution in AGP to clarify how
action-level coordination is realized at test time. The algorithm explicitly separates
(i)~construction of a global action graph over all agents’ available actions,
(ii)~message passing to compute coordination contexts, and
(iii)~decentralized action selection by local Q-agents.

\section{Computational Complexity of Action-Graph Policies}
\label{app:complexity}

We analyze the computational complexity of Action-Graph Policies (AGP) at both training
and execution time, and compare it to standard MARL baselines, including value
decomposition and coordination-graph methods. Throughout, let $N$ denote the number of
agents, $|\mathcal{A}|$ the (maximum) per-agent action space size, $d$ the embedding
dimension, $H$ the number of attention heads, and $L$ the number of message-passing layers.

\paragraph{Action-Graph Construction.}
AGP constructs a graph over all agent--action pairs, yielding
$|\mathcal{V}| = \sum_{i=1}^N |\mathcal{A}_i| \le N|\mathcal{A}|$ nodes.
This graph is shared across agents and timesteps and does not depend on episode length.
Node features are computed independently via an encoder $\phi(o_i,a)$, incurring
$\mathcal{O}(N|\mathcal{A}|\,d)$ cost per timestep.

\paragraph{Message Passing over the Action Graph.}
Each message-passing layer applies multi-head attention over action nodes.
In the fully connected case, a single layer has worst-case complexity
\[
\mathcal{O}\big(H \, |\mathcal{V}|^2 \, d\big) = \mathcal{O}\big(H \, N^2 |\mathcal{A}|^2 \, d\big).
\]
In practice, this cost is significantly reduced by (i) masking invalid actions,
(ii) restricting attention to structured neighborhoods (e.g., same-action or
same-agent edges), and (iii) batching across environments and timesteps.
With $L$ layers, the total message-passing cost scales as
$\mathcal{O}(L H N^2 |\mathcal{A}|^2 d)$ in the worst case.

While this may appear large, modern graph neural network implementations routinely scale
to graphs with millions of nodes and edges using GPU parallelism and optimized attention
kernels. Importantly, AGP’s graph is small in typical MARL settings
(e.g., $N\le 10$, $|\mathcal{A}|\le 10$), and its cost is amortized across agents, as the
action graph is processed \emph{once} per timestep rather than once per agent.

\paragraph{Coordination Context Pooling.}
After message passing, each agent pools over its own action nodes to form a coordination
context $\boldsymbol{\kappa}_i$. This step costs
$\mathcal{O}(N|\mathcal{A}|d)$ and is negligible relative to attention.

\paragraph{Local Policy or Q-Value Evaluation.}
Each agent evaluates a local network conditioned on $(o_i,\boldsymbol{\kappa}_i)$.
This incurs $\mathcal{O}(N|\mathcal{A}|d)$ cost, matching standard independent or
value-based MARL methods up to constant factors.

\paragraph{Overall AGP Complexity.}
At execution time, the dominant cost in AGP is action-graph message passing.
The total per-timestep complexity is
\[
\mathcal{O}\big(L H N^2 |\mathcal{A}|^2 d \big),
\]
with all operations fully parallelizable across nodes and heads.
Crucially, AGP does \emph{not} require joint action maximization, combinatorial search, or
sequential agent execution.

\paragraph{Comparison with Coordination Graph Methods.}
Coordination-graph methods (e.g., DCG, DICG) define pairwise utilities over agent pairs and
typically require either:
(i)~independent greedy execution (cheap but limited), or
(ii)~joint action inference via message passing (e.g., max-sum) to compute
$\arg\max_{\mathbf{a}} Q_{\mathrm{tot}}(\mathbf{o},\mathbf{a})$.
The latter has complexity exponential in the treewidth of the coordination graph and is
NP-hard for general graphs with cycles.
As a result, most practical implementations avoid exact inference at test time and revert
to independent greedy execution, reintroducing the policy-side bottleneck.

In contrast, AGP shifts coordination from inference-time optimization to
representation-time message passing. Its cost is polynomial and fixed per timestep,
independent of graph cycles or higher-order dependencies, and does not grow
exponentially with $N$.

\paragraph{Comparison with Other Baselines.}
\begin{itemize}\itemsep0pt
    \item \textbf{Independent Learners (IQL):}
    $\mathcal{O}(N|\mathcal{A}|d)$ per timestep; lowest cost but severely restricted
    expressivity.
    \item \textbf{Value Decomposition (VDN/QMIX):}
    Similar per-agent cost at execution; centralized mixing is used only during training.
    Expressivity is limited by independent execution.
    \item \textbf{Policy Factorization (MACPF/FOP):}
    Often require sequential or autoregressive action generation, incurring
    $\mathcal{O}(N)$ latency and preventing parallel execution.
    \item \textbf{AGP:}
    Higher per-timestep cost due to action-graph attention, but fully parallel,
    non-sequential, and independent of joint-action inference.
\end{itemize}

\paragraph{Practical Scalability.}
In typical MARL benchmarks, $N$ and $|\mathcal{A}|$ are small, making AGP’s overhead modest
relative to environment simulation and network forward passes.
More importantly, AGP scales favorably with coordination complexity:
it can represent dense, higher-order dependencies without incurring exponential inference
costs.
As in large-scale graph learning (e.g., social networks or molecular graphs), sparsification,
local attention, or hierarchical action graphs can further reduce complexity for larger
agent populations, which we leave to future work.

\paragraph{Summary.}
AGP trades modest polynomial computation for a strict enlargement of the executable policy
class. Unlike coordination graphs, it avoids combinatorial joint inference, and unlike
independent or factorized policies, it embeds coordination directly into the policy
representation. This makes AGP computationally feasible, parallelizable, and scalable
while resolving the policy-side bottleneck identified in the main paper.  

\end{document}